\documentclass{article}

\usepackage{xcolor}
\usepackage{soul}
    \PassOptionsToPackage{numbers, compress}{natbib}
\usepackage{algorithm}
\usepackage{multicol}
\usepackage{caption}
\usepackage{comment}
\usepackage[export]{adjustbox}
\definecolor{ourspecialtextcolor}{rgb}{0.528, 0.471, 0.701} %
\usepackage{algorithm}
\usepackage{enumitem}   
\usepackage[noend]{algpseudocode}
\DeclareCaptionSubType*{algorithm}

\algrenewcommand{\algorithmiccomment}[1]{\bgroup\hfill//~#1\egroup}
\newcommand{\ptome}{\textsc{PiToMe}\@\xspace}
\usepackage{ragged2e}

\newcommand{\scolorbox}[2]{{\setlength{\fboxsep}{2pt}\colorbox{#1}{#2}}}

\usepackage[final]{neurips_2024}
\usepackage[utf8]{inputenc} 
\usepackage[T1]{fontenc}    
\usepackage[hyperfootnotes=false]{hyperref}       
\usepackage{url}            
\usepackage{booktabs}       
\usepackage{amsfonts}       
\usepackage{nicefrac}       
\usepackage{microtype}      
\usepackage{graphicx}
\usepackage{float}
\usepackage{mathtools}
\usepackage{bbm}
\usepackage{amssymb}
\usepackage{amsmath}
\usepackage{wrapfig}
\usepackage{color}
\usepackage{hyperref}
\usepackage{url}
\usepackage{amsmath}
\usepackage{algorithm}
\usepackage{algpseudocode}
\usepackage{tabularx}
\usepackage{makecell}
\usepackage{booktabs}
\usepackage{svg}
\usepackage{tikz}
\usepackage{multirow}
\usepackage{multicol}
\usepackage{colortbl}
\usepackage{caption}
\usepackage{subcaption}
\usepackage{multirow}
\usepackage{soul}

\definecolor{LightCyan}{rgb}{0.88,1,1}
\definecolor{Gray}{gray}{0.95}

\usepackage{bm}
\DeclareMathOperator*{\argmax}{arg\,max}
\DeclareMathOperator*{\argmin}{arg\,min}
\usepackage{mathtools}

\usepackage{multirow}
\usepackage{xspace}
\usepackage{tocloft}

\newcommand\blfootnote[1]{%
  \begingroup
  \renewcommand\thefootnote{}\footnote{#1}%
  \addtocounter{footnote}{-1}%
  \endgroup
}
\usepackage[textsize=tiny]{todonotes}

\definecolor{cyan}{cmyk}{.3,0,0,0}
\definecolor{darkcyan}{rgb}{0.0, 0.55, 0.55}
\usepackage{amsthm}

\newtheorem{assumption}{Assumption}
\newtheorem{lemma}{Lemma}

\newtheorem{theorem}{Theorem}
\newtheorem{conjecture}{Conjecture}
\newtheorem{proposition}{Proposition}
\newtheorem{definition}{Definition}

\newcommand{\red}[1]{\textcolor{red}{#1}}

\newcommand{\blue}[1]{\textcolor{blue}{#1}}

\newcommand{\ps}{^{(s)}}

\newcommand{\sd}{\text{SD}}

\usepackage{bm}

\usepackage{comment}

\def\st{{\em s.t.~}}
\def\ie{{\em i.e.,~}}
\def\eg{{\em e.g.,~}}

\def\wrt{{\em w.r.t.~}}
\def\resp{{\em resp.~}}
\def\respn{{\em resp.}}
\def\wlogt{{\em w.l.o.g.~}}

\newcommand{\R}{\mathbb{R}}

\newcommand{\Pb}{\mathbb{P}} 

\newcommand{\card}{\text{card}} 
\newcommand{\cut}{\text{cut}}

\newcommand{\Ep}[1]{\mathbb{E}\left(#1\right)}

\newcommand{\cA}{\mathcal{A}}

\newcommand{\cB}{\mathcal{B}}

\newcommand{\cC}{\mathcal{C}}

\newcommand{\cE}{\mathcal{E}}

\newcommand{\cG}{\mathcal{G}}

\newcommand{\cL}{\mathcal{L}}

\newcommand{\cM}{\mathcal{M}}

\newcommand{\cN}{\mathcal{N}}

\newcommand{\cP}{\mathcal{P}}

\newcommand{\cU}{\mathcal{U}}

\newcommand{\cV}{\mathcal{V}}

\newcommand{\bfD}{\mathbf{D}}

\newcommand{\bfE}{\mathbf{E}}

\newcommand{\bfI}{\mathbf{I}}

\newcommand{\bfL}{\mathbf{L}}

\newcommand{\bfu}{\mathbf{u}}

\newcommand{\bfv}{\mathbf{v}}

\newcommand{\bfW}{\mathbf{W}}

\newcommand{\bslambda}{\boldsymbol{\lambda}}

\newcommand{\hlc}[2][yellow]{{\sethlcolor{#1} \hl{#2}} }

\title{Accelerating Transformers with Spectrum-Preserving Token Merging}

\author{Hoai-Chau Tran$^{*1,2}$, Duy M. H. Nguyen$^{*1,3,4}$,~Duy M. Nguyen$^{5}$,~TrungTin Nguyen$^{6}$,~Ngan Le$^{7}$, \\ ~\textbf{Pengtao Xie}$^{8,9}$,~\textbf{Daniel Sonntag}$^{1,10}$,~\textbf{James Zou}$^{11}$,~\textbf{Binh T. Nguyen}$^{\dagger\,2}$,~\textbf{Mathias Niepert}$^{\dagger\,3,4}$
\\\\
$^{1}$German Research Center for Artificial Intelligence (DFKI), $^{2}$University of Science - VNUHCM\\
$^{3}$Max Planck Research School for Intelligent Systems (IMPRS-IS), $^{4}$University of Stuttgart, \\$^{5}$Dublin City University,~$^{6}$University of Queensland, $^{7}$University of Arkansas,~$^{8}$MBZUAI,\\
~$^{9}$UC San Diego,~$^{10}$Oldenburg University,~$^{11}$Stanford University.
}
\begin{document}
\maketitle
\vspace{-0.2in}
\begin{abstract}
\vspace{-0.1in}
Increasing the throughput of the Transformer architecture, a foundational component used in numerous state-of-the-art models for vision and language tasks (e.g., GPT, LLaVa), is an important problem in machine learning. One recent and effective strategy is to merge token representations within Transformer models, aiming to reduce computational and memory requirements while maintaining accuracy. Prior works have proposed algorithms based on Bipartite Soft Matching (BSM), which divides tokens into distinct sets and merges the top $k$ similar tokens. However, these methods have significant drawbacks, such as sensitivity to token-splitting strategies and damage to informative tokens in later layers. This paper presents a
novel paradigm called \ptome, which prioritizes the preservation of informative tokens using an additional metric termed the \textit{energy score}. This score identifies large clusters of similar tokens as high-energy, indicating potential candidates for merging, while smaller (unique and isolated) clusters are considered as low-energy and preserved.
Experimental findings demonstrate that \ptome saved from 40-60\% FLOPs of the base models while exhibiting superior off-the-shelf performance on image classification (0.5\% average performance drop of ViT-MAEH compared to 2.6\% as baselines), image-text retrieval (0.3\% average performance drop of CLIP on Flickr30k compared to 4.5\% as others), and analogously in visual questions answering with LLaVa-7B. Furthermore, \ptome is theoretically shown to preserve intrinsic spectral properties to the original token space under mild conditions. Our implementation is available at this \href{https://github.com/hchautran/PiToMe}{link}.
\blfootnote{$^{*}$Co-first author, $^{\dagger}$Corresponding Authors.}
\end{abstract}
\vspace{-0.15in}
\addtocontents{toc}{\protect\setcounter{tocdepth}{-1}}
\section{Introduction}

\vspace{-0.1in}
Vision Transformers (ViTs)~\citep{ViT} have been integral to recent advancements in computer vision, leading to state-of-the-art deep learning architectures for representing images and videos~\citep{liu2024visual,mh2024lvm,kirillov2023segment,liu2022video}. However, these transformer-based architectures incur substantial memory costs and have a quadratic time complexity in the number of tokens due to the self-attention layers. This challenge becomes particularly severe as model sizes increase, as observed in Large Language Models (LLMs)~\citep{dehghani2023scaling}.

To address such limitations, several efforts focus on designing a more \textit{efficient attention} mechanism by making it linearly scale with input tokens~\cite{katharopoulos2020transformers,bolya2022hydra}, integrating vision or language domain-specific modules~\citep{liu2021swin,nguyen2023joint}, or pruning the head numbers in ViT ~\citep{meng2022adavit,michel2019sixteen}. Others propose \textit{dynamically pruning} less important tokens \wrt pre-defined metrics using learnable masks ~\citep{kong2022spvit,wei2023joint}. However, a primary downside of these novel methodologies lies in the necessity to retrain the model from scratch, therefore hindering the leveraging of well-trained models such as LLMs. Moreover, most pruning-based techniques may not accelerate the training process, which arises from the dynamic removal of tokens in each sample, resulting in a mismatch of dimensions and consequently preventing the batching of samples with consistent dimensions.

\begin{figure}[t]
    \centering
     \resizebox{0.95\textwidth}{!}{
    \input{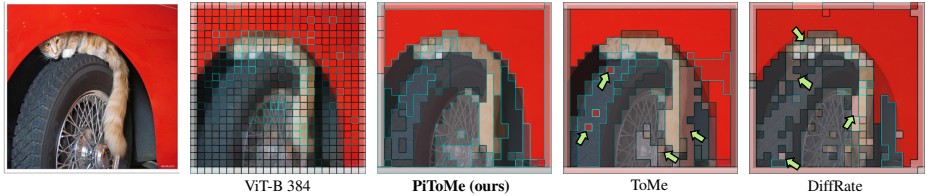}}
    \vspace{-0.05in}
    \caption{A comparison of token merging methods. Patches of the same color are merged. Green arrows highlight incorrect merges, avoided by \ptome. Position of tokens with high attention scores (\textcolor{darkcyan}{cyan} borders, zoom for clarity) in \ptome are maintained proportionality akin to ViT-base 384.}
    \vspace{-0.25in}
    \label{fig:concept}
\end{figure}
Recent research has introduced a novel \textit{token merging} technique.  Instead of pruning, this method combines tokens with high semantic similarity, removing background tokens and merging less informative foreground ones. Its versatility extends to training and non-training scenarios, drastically reducing compute and memory usage. A notable example is ToMe \citep{tome}, which introduced the Bipartite Soft Matching (BSM) algorithm, prominent for its simplicity and effectiveness in merging highly similar tokens. Since ToMe, several works, including ToFu \citep{kim2024token}, Pumer \citep{pumer}, LTPM \citep{LTPM}, and DiffRate \citep{DiffRate}, have built upon BSM with various adaptations in vision and language domains. In BSM, tokens representing image patches are separated into sets $\mathcal{A}$ and $\mathcal{B}$, and their pairwise cosine similarity is computed. The top $k$ similar pairs of tokens between the sets $\mathcal{A}$ and $\mathcal{B}$ are merged. However, the performance of this algorithm is sensitive to the token-splitting strategy. For instance, ToMe's approach, which first splits tokens based on index parity, can lead to incorrect merging since tokens in $\mathcal{A}$ can subsequently only be merged with those in $\mathcal{B}$ (Figure \ref{fig:concept}). Moreover, while BSM excels in initial layers with many redundant tokens, deeper layers risk merging informative tokens due to latent object correlations. Though current enhancements \citep{DiffRate} mitigated this by considering token attention scores in BSM \citep{jiang2022trips}, their adaptability to different ViT architectures, each with potentially distinct attention score distributions \citep{park2023self}, remains a challenge.

In this work, we propose \ptome (\textcolor{darkcyan}{P}rotect \textcolor{darkcyan}{I}nformative \textcolor{darkcyan}{To}kens before \textcolor{darkcyan}{Me}rging), a method designed to safeguard crucial information-bearing tokens prior to the merging step. Our method prioritizes preserving informative tokens by utilizing an additional metric termed the \textit{energy score} inspired by connections to \textit{graph energy} in spectral graph theory~\citep{balakrishnan2004energy,gutman2006laplacian} (Theorem \ref{theorem_spectrum_consistent_PiToME}). Specifically, our energy score assesses large clusters of similar tokens as possessing high 
energy (like background and repeated textures), thereby marking them as suitable candidates for merging, while smaller, distinct regions (foreground) are deemed low-energy and thus treated as protected informative tokens. The proposed energy term operates on the graph built for input tokens, taking into account their relationships and aggregating information from nearby neighbors when their similarities exceed certain thresholds. This approach facilitates a deeper contextual comprehension compared to previous works \citep{tome,kim2024token,pumer,DiffRate} that rely solely on attention scores or feature embedding per token. Subsequently, we only select the highest-scoring tokens and pass them on for merging in the next steps, ensuring the preservation of important tokens, particularly in the latter stages when only a few remaining ones. During the merging process, we continue leveraging sorted energy vectors from earlier stages by distributing tokens with similar energy into two sets, $\mathcal{A}$ and $\mathcal{B}$, resulting in candidates in $\mathcal{A}$ having a high probability of finding compatible matches in $\mathcal{B}$. Matched tokens are then merged using a weighted average feature embedding to create a new token representation. 

The empirical results demonstrate that despite the increased computational cost associated with energy score calculations, \ptome exhibits comparable speed to other BSM-based approaches since the matching is performed on a smaller, high-energy token set. At the same time, it consistently shows superior accuracy across various experimental scenarios. Additionally, we present theoretical insights into \ptome, showing that, under moderate assumptions — such as the discriminative nature of feature embeddings generated by ViT for node pairs within and across distinct objects — our algorithm efficiently preserves the spectral properties of the initial input tokens, maintaining the eigenvalues derived from normalized Laplacian matrices of the original tokens~\citep{loukas2018spectrally,jin2020graph,loukas2019graph}.
To summarize, our contributions encompass:
\vspace{-0.1in}
\begin{itemize}[noitemsep, nolistsep]
\vspace{0.05in}
    \item A new token merging procedure for accelerating ViT architectures is designed to protect crucial yet small-region tokens while identifying redundant ones for merging based on contextual token correlations captured by our energy score functions. 
    \vspace{0.05in}
    \item Our \ptome runs as fast as other BSM-based approaches while achieving SOTA performance on diverse tasks, ranging from image-text retrieval (Sec.~\ref{sec:img-text}), visual question answering with LLMs (Sec.~\ref{sec:vqa-llm}), image classification (Sec.~\ref{subsec:ic}), and text classification (Sec.~\ref{sec:text-class}). In several cases, \ptome is shown to reduce up to $40-60\%$ FLOPs of base models while only dropping performance around $~0.3-0.5\%$ (CLIP model on Flick30k).
     \vspace{0.05in}
    \item We also present theoretical findings indicating that, given reasonable assumptions, \ptome can effectively approximate the spectral distance between the initial token spaces and the merged token set. This sheds light on why \ptome tends to outperform baselines in practical applications and contributes to a better understanding of the potential limitations inherent in BSM-based methods, such as those in \citep{tome,kim2024token,DiffRate,pumer,wang2023zero}.
\end{itemize}




\vspace{-0.05in}
\section{Related Work}
\vspace{-0.1in}
\textbf{Efficient Attention Mechanisms.} Various efforts have sought to enhance the efficiency of transformers in both NLP and Vision domains. Some concentrate on accelerating attention computation \citep{shen2021efficient,dao2022flashattention,bolya2022hydra} through approximation techniques involving hashing \citep{daras2020smyrf}, low-rank \citep{likhosherstov2021sub}, or sparse approximations \citep{ren2021combiner}. Others explore strategies such as head or feature pruning \citep{meng2022adavit,fayyaz2022adaptive} or the integration of domain-specific modules \citep{liu2021swin,liu2022video,wang2023internimage,nguyen2023joint}.  However, many of them necessitate joint training with the backbone model from scratch. For instance, DynamicViT \citep{rao2021dynamicvit} runs approximately 150 hours of fine-tuning on an NVIDIA A100 GPU to prune the DeiT-S model \citep{touvron2021training}. In contrast, we focus on accelerating existing ViT models by token merging, which applies to training and non-training scenarios.

\textbf{Dynamic Token Pruning.} Several studies have explored token pruning in transformer models across NLP 
\citep{goyal2020power,zhong2023revisiting,yun2023focus} 
and vision domains \citep{yin2022vit,spvit,song2022cp,wang2023zero}. However, like efficient transformers, these methods typically require training. Additionally, most pruning techniques are dynamic, meaning the number of tokens varies across different inputs, which improves accuracy but complicates batching for practical deployment. To address this, numerous pruning methods employ masks during the training phase rather than directly eliminating tokens; however, it yields to cancel out the speed advantages associated with pruning.

\textbf{Token Merging.} Leading techniques such as ToMe \citep{tome} and its improvements \citep{pumer, TokenMF, LTPM, DiffRate, kim2024token, shi2023crossget}, build upon lightweight Bipartite Soft Matching (BSM). These methods exhibit speeds comparable to pruning while achieving superior performance. They have demonstrated the ability to double the throughput of state-of-the-art Vision Transformers (ViT) on both images and videos with minimal accuracy degradation in various scenarios. However, BSM-based approaches are sensitive to the selection of sets in the matching process, potentially resulting in the loss of informative tokens due to heuristic merging procedures. To address these issues, methods like DiffRate \citep{DiffRate} and Crossget \citep{shi2023crossget} leverage attention scores in ViT or cross-modal guidance to identify important tokens during the matching process, though they remain sensitive to the distribution of the token space, especially with imbalanced clusters. Another direction involves adapting more intricate algorithms, such as k-means \citep{marin2023token}, spectral clustering \citep{bianchi2020spectral}, graph pooling \citep{wu2022structural}, or graph coarsening \citep{loukas2018spectrally,kumar2023featured}, to merge similar tokens. While these strategies offer some guarantees and well-controlled outputs, their iteration schemes are highly complex and may not align with the goal of reducing model complexity in ViT layers. Our \ptome, on the other hand, enables the advantages of both approaches. It maintains efficiency comparable to BSM, remains robust to token partitioning strategies, and offers a reasonable trade-off between speed and accuracy. 
Moreover, \ptome is theoretically proved to approximate the spectral spectrum of the original token space under reasonable assumptions, resembling the behavior of other spectral clustering methods.

\vspace{-0.1in}
\section{Methodology}
\begin{figure}[t]
  \centering
    \resizebox{0.9\textwidth}{!}{
    \input{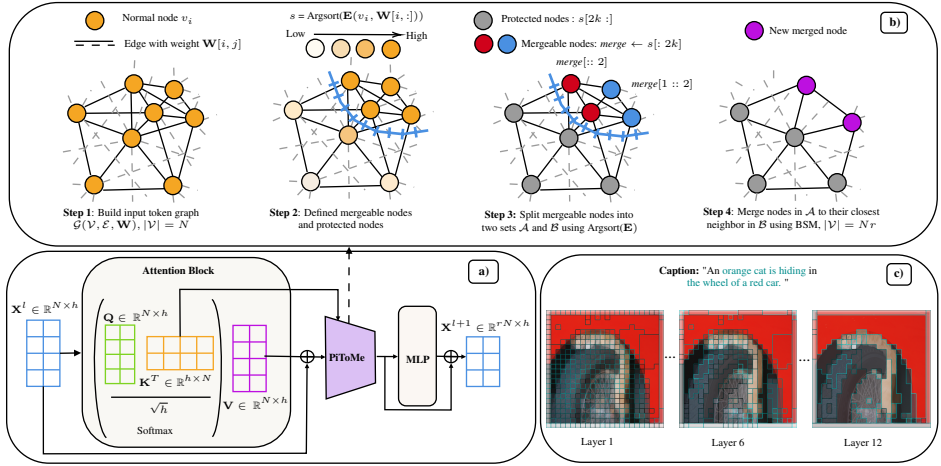}}
  \caption{\textbf{a)} \ptome can be inserted inside transformer block; \textbf{b)} Energy scores are computed to identify mergeable and protective tokens; \textbf{c)} Our algorithm gradually merges tokens in each block. 
  }
  \vspace{-0.2in}
  \label{fig:method}
\end{figure}
\vspace{-0.1in}
\subsection{Token Merging Formulation}
\vspace{-0.1in}
We apply token merging to each transformer block of the ViT architecture (Figure \ref{fig:method}-a). Given the input token of the $l$-th block $\mathbf{X}^l\in\mathbb{R}^{N\times h}$ where $N$ and $h$ are the token length and token hidden embeddings, a forward step in one Transformer block can be formulated as follows:
\begin{equation}
    \mathbf{\hat{X}}^l= \mathbf{X}^l + \text{Attention}( \mathbf{X}^l\mathbf{W}_Q,\mathbf{X}^l\mathbf{W}_K,\mathbf{X}^l\mathbf{W}_V), \text{ }\mathbf{X}^{l+1} = \mathbf{\hat{X}}^l + \text{MLP}(\mathbf{\hat{X}}^l) \label{eq:x_att}
\end{equation}
where $\text{Attention}$ and $\text{MLP}$ are the self-attention and multiple layer perceptron components. We then apply merge operations on $\mathbf{\hat{X}}^l$ and compute the output of the reduced MLP block as:
\begin{equation}
    \mathbf{X}^{l+1} = \mathbf{\hat{X}}^l_m + \text{MLP}(\mathbf{\hat{X}}^l_m), \text{ where }\mathbf{\hat{X}}^l_m=\mathbf{F}_{\mathrm{mer}}(\mathbf{\hat{X}}^l, \mathbf{X}^l\mathbf{W}_K, r).
\end{equation}
Here $\mathbf{F}_{\mathrm{mer}}(.)$ is the merging operation that receives $\mathbf{\hat{X}}^l$ as input for compressing, $\mathbf{X}^l\mathbf{W}_K$ (key matrices) as the token features of $\mathbf{\hat{X}}^l$ following prior work \citep{tome, TokenMF,LTPM, DiffRate}, and $r$ is the fraction 
of remaining tokens. The output $\mathbf{\hat{X}}^l_m \in \mathbb{R}^{rN \times h}$ serves as input for the MLP layer to produce $\mathbf{X}^{l+1} \in \mathbb{R}^{rN\times h}$. We present the \ptome $\mathbf{F}_{\mathrm{mer}}(.)$ function in the next section.


\subsection{Energy-based Merging}\label{sec:method}
\vspace{-0.1in}
We propose to use a new term called \textit{energy score} to evaluate the redundancy of each token, which is then used to protect informative or isolated tokens (low energy scores) while considering tokens that
are in the large cluster as high energy scores and characterizing them as merging candidates. Figure \ref{fig:method}-b illustrates the main steps in \ptome.
\label{subsec:overall}

\textbf{Token Graph Construction}: Given a set of $N$ token inputs in $\mathbf{\hat{X}}^l$, we build a weighted graph $\cG(\cV,\cE,\bfW)$ with $\cV$ a set of $N = |\cV|$ nodes, $\cE$ a set of $M = |\cE|$ edges
defined by connecting one token to the remaining ones in $\cG$, $\bfW \in \R^{N\times N}$ be a weighted adjacency matrix. We opt for using the \textit{key} vectors $\mathbf{K} = \mathbf{X}^l\mathbf{W}_K \in \mathbb{R}^{N\times h}$ as node features of $\cV$, i.e., $v_i \in \mathcal{V}$ has $h$ feature dimensions. The weight $\bfW[i,j]$ assigned to an edge $e_{ij} \in \cE$ connects $v_i$ and $v_j$ is computed by cosine distance:
\begin{align}\label{eq_weight_def}
\bfW[i,j] = 1-\cos(v_i,v_j), \text{ where } \cos(v_i,v_j) = \frac{v_i\cdot v_j}{{\|v_i\| \|v_j\| }}, \quad \forall v_i \in \cV, v_j \in \cV. 
\end{align}
For simplicity, $\bfW[i,:]$ and $\bfW[:,i]$ denote the i-th row and column, \respn; $[N]$ stands for $\{1,\ldots,N\}$.

    \textbf{Token Energy Scores}: 
    In this step, the \textit{energy score}, denoted as $\mathbf{E} = (E_i)_{i \in [N]}$, is computed for each node  (Figure \ref{fig:method}-a, Step 2). The term is inspired by the concept of \textit{graph energy} in spectral graph theory \citep{balakrishnan2004energy,gutman2006laplacian}, defined as the sum of the absolute eigenvalues of the adjacency matrix $\mathbf{W}$. We also leverage such structures of $\mathbf{W}$ to find correlations among tokens and to estimate token redundancy. Instead of using independent token values such as attention scores \citep{DiffRate}, our energy leads to better performance (Figure \ref{fig:ablation}, Appendix) and provides theoretical connections to the spectral properties of the original token graphs (Theorem \ref{theorem_spectrum_consistent_PiToME}).
        
    Let $i$ be the index of the current node and $\mathcal{N}(i)$ represent the set of neighbor nodes. The energy score $E_i \equiv E_i(v_i,\bfW[i,:])$ of node $v_i$ is calculated using the following equation:
    \begin{equation} \label{equa:energy}
       E_i(v_i,\bfW[i,:]) = \frac{1}{N} \sum_{j \in \mathcal{N}(i)}  {f}_m(\cos(v_i,v_j)), \text{ } {f}_m(x) = \begin{cases}
            x  & \text{ if }  x \geq m \\
            \alpha(\exp(x-m) - 1) & \text{otherwise}
        \end{cases}. 
    \end{equation}
Rather than accumulating all $\cos(v_i, v_j)$ values, the function $f_m(.)$ in Eq.\eqref{equa:energy} mimics the exponential linear unit activation function \cite{clevert_fast_2016}, focusing on similar tokens even if they are far apart, while ignoring dissimilar ones.
Here, $m$ is a dynamic margin value varying at each layer in the ViT model. Nodes within this margin, i.e., ($x > m$) with high cosine similarity $\cos(v_i, v_j)$ are considered true neighbors, potentially representing tokens belonging to the same object. Nodes outside this margin have $\cos(v_i, v_j)$ replaced by a constant $\alpha$, providing a lower bound for minimal edge weights. The term $\exp(x-m) - 1 < 0$ smooths the function $f(x)$ for neighboring nodes near the margin $m$. In experiments, we set $\alpha = 1.0$ and $m = 0.9 - 0.9 \times l_i / l$, where $l_i$ is the current layer index and $l$ is the total number of encoder layers, indicating an increasing margin as tokens move to deeper layers. The ablation studies for the $\alpha$ and $m$ values are presented in Section \ref{sec:schedule}.

Intuitively, Eq.\eqref{equa:energy} reflects the number of tokens potentially representing the same object. Tokens belonging to large objects (e.g., background) will have high energy scores, indicating potential candidates for merging, while smaller ones (e.g., foreground) will have low energy scores and are considered to be protected. This guides us to sort the energy vectors $\mathbf{E}$ in descending order and choose only the top $2k$ nodes with the highest scores as mergeable candidates and the remaining ones as protective tokens, i.e, 
$\mathbf{s} = \mathrm{argsort}(\mathbf{E})$, $\mathrm{\textit{merge}} \leftarrow \mathbf{s}[:2k],\ \mathrm{\textit{protect}} \leftarrow \mathbf{s}[2k:], k = N - Nr$.
    
\textbf{Ordered Energy-based Bipartite Soft Matching}: Having identified mergeable tokens in the $\mathrm{\textit{merge}}$ set, we continue exploit the sorted order in $\mathbf{E}$ to form two sets $\mathcal{A}$ and $\mathcal{B}$ in BSM, each containing $k$ nodes.
Specifically, tokens with odd and even indices in $\mathrm{\textit{merge}}$ are selected for $\mathcal{A}$ and $\mathcal{B}$, \respn \,given 
the fact that those in the same object should have similar energy scores, resulting in likely distributing 
in consecutive positions in $\mathrm{argsort}(\mathbf{E})$. In other words, our choosing has a high probability that one token in $\mathcal{A}$ always finds its best match in the same object in $\mathcal{B}$. This sets us apart with random partitions based on spatial indices in images like \citep{tome,kim2024token}.

\textbf{Tracking Token Sizes} All nodes in set $\mathcal{A}$ are then merged with their nearest neighbors in set $\mathcal{B}$ through the \textit{fast} BSM algorithm. 
Following prior works \citep{tome,kim2024token}, we also add proportional attention to balance the effect of the merged token on the output of the softmax function:
$\displaystyle{\mathbf{A} = \mathrm{Softmax}\left(\mathbf{X}^l\mathbf{W_Q}\cdot(\mathbf{X}^l\mathbf{W_K})^{T} / \sqrt{h} + \log \mathbf{m}\right)}$ where $\mathbf{m}$ is a row vector containing the size of each token, i.e., the number of data patches the token represents. The pseudo-code for our method is provided in Algorithm \ref{algo:pitome} (Appendix) with complexity analysis. 
\vspace{-0.05in}
\subsection{Connection to Graph Coarsening with Spectral Preservation}
\vspace{-0.1in}
In this section, we employ tools from spectral graph theory to show a spectral distance preservation of \ptome. We note that similar properties can be obtained by using more complicated clustering algorithms such as K-mean \citep{marin2023token} or spectral clustering \citep{bianchi2020spectral,wu2022structural,loukas2018spectrally}; however, these methods are typically loop-based algorithms, which are computationally expensive and not suitable for batch-type data. Our \ptome, in contrast, is as fast as BSM methods but theoretically preserves spectral properties of input token graphs. 

We begin by introducing Definitions \ref{def_Gc} and \ref{def_Gl} of graph coarsening  and lifting, \respn, 
to justify the spectral distance constructed in equation (\ref{eq_def_SD}), measuring the similarity between the original and coarse graphs.
For more thorough coverage of the mathematics of graph coarsening and graph lifting, we refer the reader to~\citep{jin_graph_2020,loukas_graph_2019,brissette_spectrum_2023,loukas_spectrally_2018}. In short, \textit{we treat the result of token merging as a graph coarsening process} (Figure~\ref{fig:graph-coarsen}, Appendix). We then create the \textit{lifted graph} as a reconstruction from this coarsened version to assess the spectral distance to the original token graph.
\vspace{0.05in}
\begin{definition}[Graph Coarsening]\label{def_Gc}
    Given a weighted graph $\cG(\cV,\cE,\bfW)$, we denote  $\cP = \{\cV_i\}_{i\in[n]}$ where $\cV = \cup_{i\in[n]}\cV_i$, be a partition of its node into $n$ disjoint sets.
    The coarsened graph of \(\mathcal{G}\) \wrt\(\mathcal{P}\) is the weighted graph \(\mathcal{G}_c\), where each partition in \(\mathcal{P}\) is aggregated into a single node, denoted \(\{\nu_i\}_{i \in [n]}\), by averaging the elements within each partition. The elements of the adjacency matrix are given by $\bfW_c[i,j] = \sum_{v_i\in\cV_i}\sum_{v_j \in \cV_j} \bfW[i,j]/(|\cV_i||\cV_j|)$.
    We denote the combinatorial and normalized Laplacians of $\cG$ by $\bfL = \bfD - \bfW$ and $\cL = \bfI_N - \bfD^{-1/2}\bfW\bfD^{-1/2}$, \respn, where $\bfD$ is the diagonal degree matrix with $\bfD[i,i] = d_i := \sum_{j=1}^N \bfW[i,j]$.
    Similarly, the definition of the coarsened Laplacian matrices follows directly: $\bfL_c = \bfD_c - \bfW_c$ and $\cL_c = \bfI_n - \bfD_c^{-1/2}\bfW_c\bfD_c^{-1/2}$.
    Finally, the eigenvalues and eigenvectors of $\cL$ (\resp $\cL_c$) are denoted as $\bslambda$ and $\bfu$ (\resp $\bslambda_c$ and $\bfu_c$).
\end{definition}
\vspace{0.05in}
\begin{definition}[Graph Lifting]\label{def_Gl}
     We call $\cG_l(\cV_l,\cE_l,\bfW_l)$ the lifted graph of $\cG$ if the adjacency matrix elements are given by $\bfW_l[i,j] = \bfW_c[i,j]$.
    We denote the node degree of $v_{li} \in \cV_l$ by $d_{li} = \sum_{j=1}^N \bfW_l[i,j]$. The combinatorial and normalized Laplacians of $\cG_l$ is then defined as $\bfL_l = \bfD_l - \bfW_l$ and $\cL_l = \bfI_N - \bfD_l^{-1/2}\bfW_l\bfD_l^{-1/2}$, \respn, where $\bfD_l$ is the diagonal degree matrix with $\bfD[i,i] = d_{li}$.
    Then, we denote, \respn, the eigenvalues and eigenvectors of $\cL_l$ by $\bslambda_l$ and $\bfu_l$.
\end{definition}
\begin{lemma}[Eigenvalue Preservation, see \eg~\cite{jin_graph_2020,loukas_graph_2019,toivonen_compression_2011,butler_interlacing_2007}]\label{lem_Eigenvalue_Preservation}
    The normalized Laplacian eigenvalues of the lifted graph $\bslambda_l$ contain all the eigenvalues of the coarse graph $\bslambda_c$ and additional eigenvalues $1$ with $(N-n)$ multiplicity.
\end{lemma}
Through Lemma \ref{lem_Eigenvalue_Preservation}, we can use the lifted graph $\cG_l$ as a proxy for the coarse graph $\cG_c$, and define: 
\vspace{-0.05in}
\setlength{\abovedisplayskip}{2pt}
\setlength{\belowdisplayskip}{2pt}
\begin{align}\label{eq_def_SD}
    \sd(\cG,\cG_c) = \|\bslambda-\bslambda_l\|_1 = \sum_{i=1}^N |\lambda_i-\lambda_{li}| \text{ as a spectral distance.}
    \vspace{-0.05in}
\end{align}
Next, we present our main theoretical result demonstrating how spectral distance characterizes the superiority of our novel \ptome paradigm over the state-of-the-art approaches as ToMe \citep{tome,kim2024token}.
The Theorem \ref{theorem_spectrum_consistent_PiToME} quantifies how similar the original $\cG$ is to its coarsened counterpart $\cG_c$, and is proved in Appendix~\ref{section_proof_theorem_spectrum_consistent_PiToME}.
\begin{theorem}[Spectrum Consistent of Token Merging]\label{theorem_spectrum_consistent_PiToME}
    Suppose the graphs $\cG\ps_0$, $\cG_{\text{\ptome}}\ps$, and $\cG_{\text{ToMe}}\ps$ are coarsened from the original graph $\cG$ by iteratively merging pairs of nodes $v_{a_s}$ and $v_{b_s}$  \wrt the true partition  $\cP\ps_{0}=\{\cV\ps_{0i}\}_{i \in [s]}$, the \ptome-partition $\cP\ps_{\text{\ptome}}=\{\cV\ps_{\text{\ptome}i}\}_{i\in[s]}$, defined by \ptome in Algorithm \ref{algo:pitome}, and the ToMe-partition~\citep{tome,kim2024token}, $\cP\ps_{\text{ToMe}}=\{\cV\ps_{\text{ToMe}i}\}_{i\in[s]}$, for $s = N, \ldots,n+1$.
    We assume some standard mild assumptions: (\textit{A1}) $\mathbb{E}[\cos(v_{a_s}, v_{b_s})] \rightarrow 1, \quad \forall v_{a_s} \in \cV^{(s)}_{0i}, \forall v_{b_s} \in \cV^{(s)}_{0i}, i \in [s]$; (\textit{A2}) there exists a margin $m$ s.t.,  $\cos(v_{a_s}, v_{b_s}) \ge m > \cos(v_{a_s}, v_{c_s}), \quad \forall v_{a_s} \in \cV^{(s)}_{0i}, \forall v_{b_s} \in \cV^{(s)}_{0i}, \forall v_{c_s} \in \cV^{(s)}_{0j}, \forall i \neq j \in [s]$; and (\textit{A3}) there is an order of cardinality in the true partition, without loss of generality, we assume $N\ps_1 \ge N\ps_2 \ge \ldots \ge N\ps_s$, where $N\ps_i = |\cV^{(s)}_{0i}|, \forall i \in [s]$.
    Then it holds that:
    \vspace{-0.1in}
    \begin{enumerate}
        \item The spectral distance between the original $\cG\equiv\cG^{(N)}_{0}$ and the \ptome-coarse $\cG_{\text{\ptome}}^{(n)}$ graphs converges to $0$, \ie$\sd(\cG,\cG_{\text{\ptome}}^{(n)}) \rightarrow 0$, 
        \vspace{-0.1in}
        \item The spectral distance between the original $\cG$ and the ToMe-coarse $\cG_{\text{ToMe}}^{(n)}$ graphs converges to a non-negative constant $C$, with a high probability that $C > 0$. 
    \end{enumerate}
\end{theorem}
Intuitively, Theorem \ref{theorem_spectrum_consistent_PiToME} states that, given assumptions (i) tokens are closely embedded within classes and distinct between classes $(\textit{A1, A2})$, and (ii) the number of tokens per class follows certain orders $(A3)$, the spectral distance between \ptome and the original tokens in Eq.\eqref{eq_def_SD} will converge to $0$. In contrast, with ToMe partitions, a non-eliminable constant likely remains.
\vspace{-0.1in}
\section{Experiments}
\vspace{-0.15in}
We focus on two settings: \hlc[cyan!30]{\textit{Off-the-Shelf Performance}}, where we evaluate the models' performance immediately after compression without training, and \hlc[gray!15]{\textit{Retrained}}, where we treat the compression algorithms as pooling functions and retrain the models on downstream tasks. The experiments cover four tasks: (i) \textit{image \& text retrieval}, (ii) \textit{visual question answering (VQA)}, (iii) \textit{image classification}, and (iv) \textit{text classification}. We use the number of floating-point operations (FLOPS) needed for inference on one sample as the main metric to benchmark memory footprint and speed. Higher FLOPS indicate greater memory requirements and longer training and inference times.
\vspace{-0.1in}
\subsection{Image \& Text Retrieval}
\label{sec:img-text}
\vspace{-0.1in}
We evaluate \ptome on the image-text retrieval task using three different backbone models CLIP \citep{clip}, ALBEF \citep{ALBEF}, and BLIP \citep{BLIP} on two frequently used Flickr30k \citep{FLickr30k} and MSCOCO \citep{MSCOCO} datasets. Our experiment is benchmarked using $\mathrm{recall@k}$ \citep{cao2022image}, where a higher $\mathrm{recall@k}$ indicates the model's effectiveness in retrieval.
In Figure \ref{fig:off-the-shell-image-text}, we benchmarked \ptome against other SOTA \textit{merging} or \textit{pruning}-based methods such as ToMe~\citep{tome}, ToFu~\citep{kim2024token}, DiffRate~\citep{DiffRate}, and DCT~\citep{DCT} on \textit{off-the-shelf} setting when varying amount of merged tokens at each layer. Given the same FLOPS, it is clear that \ptome consistently outperforms previous compression algorithms across all backbones. The performance gap increases as we decrease the percentage $r$ of tokens retained in each layer. The same behavior remains consistent in Table \ref{tab:re-trained-img-text}, where we set $r=0.925$ and \textit{retrain} pre-trained checkpoints of BLIP and CLIP. For more details about the training method, please refer to \citet{BLIP}.




\begin{table}[t]
\centering

\end{table}
In Table \ref{tab:sota-img-text}, we compare \ptome using compression ratios of $r \in \{0.95, 0.975\}$ on BLIP and BLIP-2 against other advanced architectures such as ViLT~\citep{kim2021vilt}, LightningDOT~\citep{sun2021lightningdot}, UNITER~\citep{chen2020uniter}, METER~\citep{dou2022empirical}, CLIP-L~\citep{radford2021learning}, and ALBEF~\citep{li2021align}. The results show that \ptome consistently surpasses those architectures by a significant margin. Moreover, the performance drop on the base BLIP/BLIP-2 is minimal while achieving substantial reductions in memory footprint and \textit{rerank times—nearly halving for BLIP} and \textit{tripling for BLIP-2}. Additionally, the speedup can further improve with increased batch and model sizes.
\vspace{-0.1in}
\subsection{Visual Question Answering (VQA) with Large Vision-Language Models}
\label{sec:vqa-llm}
\vspace{-0.1in}
This experiment focuses on assessing the off-the-shelf performance of large vision-language models like LLaVa \citep{liu2024visual}. We extensively conduct experiments across six VQA datasets: VQA-v2~\citep{vqa_v2}, GQA~\citep{gqa} (academic questions), VizWiz~\citep{vizwiz} (visually impaired individuals), ScienceQA~\citep{scienceqa} (zero-shot scientific question answering), TextVQA \citep{textvqa} (text-rich VQA tasks), and MME-Perception~\citep{mme} (visual perception with yes/no question). More details on the number of samples in each dataset are in the Appendix.
All experiments are conducted using 
\textit{LLAVA-1.5 7B} and \textit{LLAVA-1.5 13B}
with the  
$\mathrm{lmms\_eval}$ 
library \citep{lmms_eval2024} provided by the LMMs-Lab team.
\vspace{-0.1in}
\begin{figure}[H]
\centering
     \includegraphics[width=0.82\textwidth]{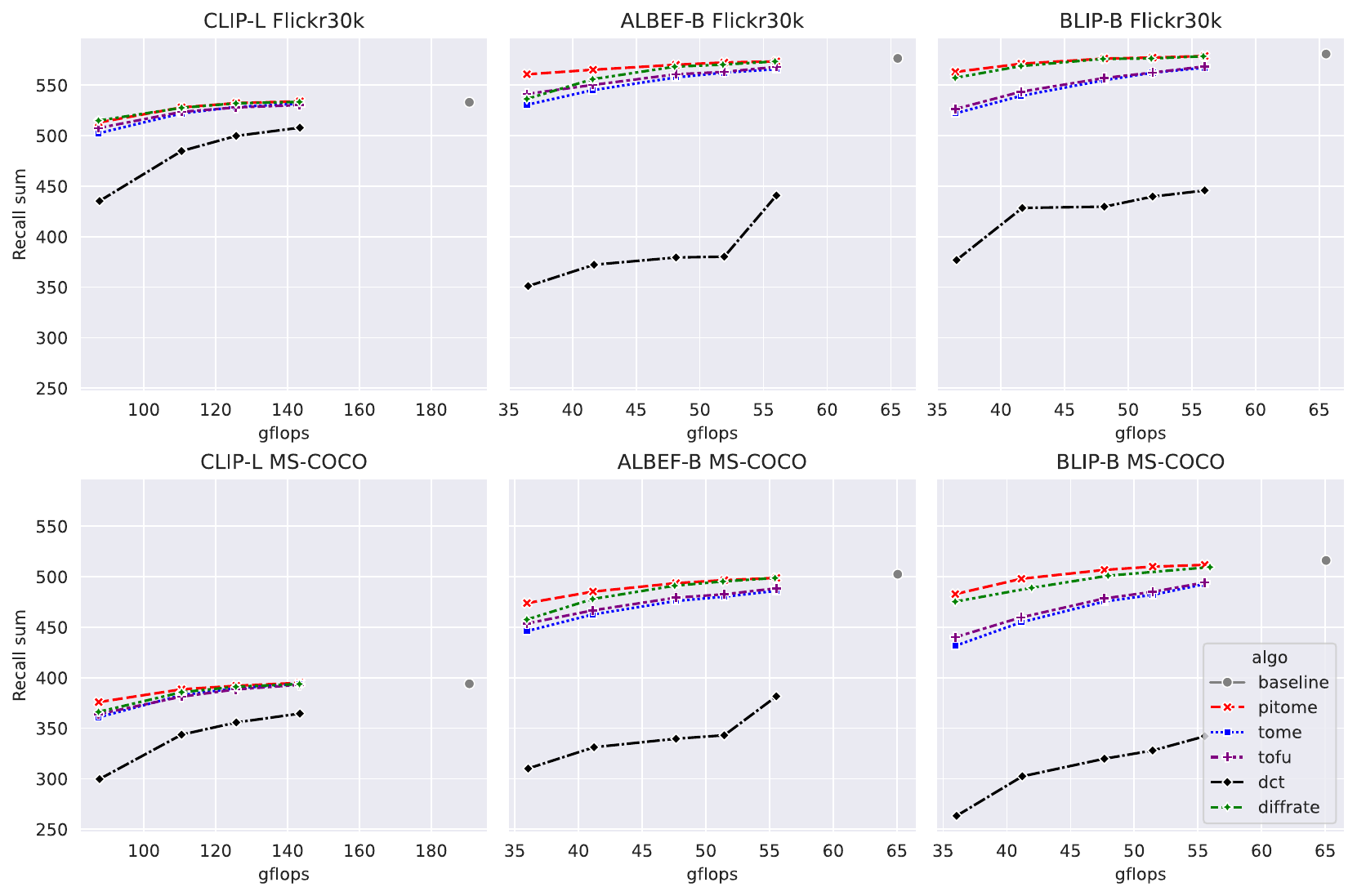}
     \vspace{-0.05in}
    \caption{\scolorbox{cyan!30}{\textbf{Off-the-shell}} \textbf{ Image-Text Retrieval comparison} between \ptome v.s. merging/pruning methods on different backbones on \textbf{} tasks when varying the number of merged tokens. Here, Recall sum =$Rt@1+ Rt@5+Rt@10+Ri@1+ Ri@5+Ri@10$ is close to 600, indicating recall scores at top 1,5, and 10 for retrieving image and text reached close to 100\%. \ptome curves, in most cases, are above other baselines.}
    \label{fig:off-the-shell-image-text}
\end{figure}
\vspace{-0.4in}
\begin{table}[H]
\begin{minipage}[ht]{.62\linewidth}
 \centering
\caption{\textbf{Image-Text Retrieval comparison.} 
\ptome without training are in \scolorbox{cyan!30}{blue}, and with training in \scolorbox{gray!15}{gray}. \ptome achieves SOTA while saving $36\%-56\%$ in FLOPS and speeding up by $\times1.4$ to $\times 1.6$ compared to the base models. 
} 
 \resizebox{\textwidth}{!}{
 \setlength{\tabcolsep}{1.8pt}
\begin{tabular}{l|l|cc|cc|cc|cc}
\toprule
\textbf{Datasets} & \textbf{Methods} &  \makecell{$\mathbf{Rt}@1$ $\uparrow$} & \makecell{$\mathbf{Ri}@1$ $\uparrow$} & \makecell{\textbf{ZS Retrieval} \\ $Rsum$ $\uparrow$} & \makecell{\textbf{Reranked}\\$Rsum$ $\uparrow$} & \makecell{\textbf{ViT} \\\textbf{FLOPS} $\downarrow$} & \makecell{\textbf{Total} \\\textbf{FLOPS} $\downarrow$} & \makecell{ZS \textbf{Retrieval} \\ \textbf{Time} $\downarrow$ }& \makecell{\textbf{Total} \\ \textbf{Time} $\downarrow$}\\
\toprule
\multirow{11}{*}{Flickr30k} 
            &ViLT&83.50&64.40&490.60&525.70&-&55.90&-&-\\
            &LightingDOT&83.90&69.90&532.26&-&-&-&-&-\\
            &UNITER&92.87&83.73&521.90&542.80&-&949.9&-&-\\
            &METER&94.30&82.22&560.54&570.72&-&-&-&-\\
            &CLIP-L&92.90&81.34& 568.23&-&80.85&-&25s&-\\
            &ALBEF&94.91&85.32& 564.58&575.00&55.14&65.54&16s&58s\\
\cline{2-10}
&\cellcolor{cyan!30}$\text{\textbf{\ptome}}^{\text{BLIP}}_{r=0.95}$ &\cellcolor{cyan!30}95.72&\cellcolor{cyan!30}86.32&\cellcolor{cyan!30}567.58&\cellcolor{cyan!30}577.81&\cellcolor{cyan!30}38.55&\cellcolor{cyan!30}47.65&\cellcolor{cyan!30}13s&\cellcolor{cyan!30}56s\\ 
&\cellcolor{gray!15}$\text{\textbf{\ptome}}^{\text{BLIP}}_{r=0.95}$&\cellcolor{gray!15}\textbf{96.61}&\cellcolor{gray!15}\textbf{87.18}&\cellcolor{gray!15}\textbf{569.98}&\cellcolor{gray!15}\textbf{579.35}&\cellcolor{gray!15}\textbf{38.55}&\cellcolor{gray!15}\textbf{47.65}&\cellcolor{gray!15}\textbf{13s}&\cellcolor{gray!15}\textbf{56s}\\
            &\text{\color{gray!70}BLIP}&\text{\color{gray!70}96.86}&\text{\color{gray!70}87.48}& \text{\color{gray!70}572.24}&\text{\color{gray!70}580.76}&\text{\color{gray!70}55.14}&\text{\color{gray!70}65.54}&\text{\color{gray!70}16s}&\text{\color{gray!70}1m17s}\\ [1mm]
            
\cline{2-10}
            &\cellcolor{cyan!30}$\text{\textbf{\ptome}}^{\text{BLIP2}}_{r=0.95}$& \cellcolor{cyan!30}96.83&\cellcolor{cyan!30}87.84&\cellcolor{cyan!30}566.25&\cellcolor{cyan!30}580.77&\cellcolor{cyan!30}296.93&\cellcolor{cyan!30}390.77&\cellcolor{cyan!30}45s&\cellcolor{cyan!30}1m21s\\
            &\cellcolor{cyan!30}$\text{\textbf{\ptome}}^{\text{BLIP2}}_{r=0.975}$&\cellcolor{cyan!30}\underline{\textbf{97.55}}&\cellcolor{cyan!30}\underline{\textbf{89.04}}& \cellcolor{cyan!30}\underline{\textbf{572.81}}&\cellcolor{cyan!30}\underline{\textbf{583.72}}&\cellcolor{cyan!30}\underline{\textbf{434.50}}&\cellcolor{cyan!30}\underline{\textbf{564.78}}&\cellcolor{cyan!30}\underline{\textbf{1m5s}}&\cellcolor{cyan!30}\underline{\textbf{1m54s}}\\  [0.5mm]
            &\text{\color{gray!70}BLIP2}&\text{\color{gray!70}97.61}&\text{\color{gray!70}89.79}& \text{\color{gray!70}572.72}&\text{\color{gray!70}584.76}& \text{\color{gray!70}678.45}&\text{\color{gray!70}900.77}&\text{\color{gray!70}1m37s}&\text{\color{gray!70}3m15s}\\ [1mm]
\midrule
\multirow{8}{*}{MS-COCO} 
        &ViLT&61.50&42.70& 420.20&439.20&-&55.90&-&-\\
        &CLIP-L&70.78&53.79& 478.18&-&80.85&-&2m10s&-\\
        &METER&76.16&57.08&-&495.95&-&-&-&-\\
        &ALBEF&76.94&60.24&478.39&500.44& 55.14&65.54&43s&5m29s\\
\cline{2-10}
&\cellcolor{cyan!30}$\text{\textbf{\ptome}}^{\text{BLIP}}_{r=0.95}$&\cellcolor{cyan!30}79.46&\cellcolor{cyan!30}62.50&\cellcolor{cyan!30}485.99&\cellcolor{cyan!30}506.65&\cellcolor{cyan!30}38.85&\cellcolor{cyan!30}47.65&\cellcolor{cyan!30}51s&\cellcolor{cyan!30}4m30s\\
&\cellcolor{gray!15}$\text{\textbf{\ptome}}^{\text{BLIP}}_{r=0.95}$&\cellcolor{gray!15}\textbf{80.44}&\cellcolor{gray!15}\textbf{63.91}&\cellcolor{gray!15}\textbf{493.33}&\cellcolor{gray!15}\textbf{512.66}&\cellcolor{gray!15}\textbf{38.85}&\cellcolor{gray!15}\textbf{47.65}&\cellcolor{gray!15}\textbf{51s}&\cellcolor{gray!15}\textbf{4m30s}\\
        &\text{\color{gray!70}BLIP}&\text{\color{gray!70}81.82}&\text{\color{gray!70}64.36}&   \text{\color{gray!70}494.34}&\text{\color{gray!70}516.03}&\text{\color{gray!70}55.14}&\text{\color{gray!70}65.54}&\text{\color{gray!70}1m3s}&\text{\color{gray!70}7m10s}\\[1mm]
\cline{2-10}
&\cellcolor{cyan!30}$\text{\textbf{\ptome}}^{\text{BLIP-2}}_{r=0.95}$&\cellcolor{cyan!30}82.29&\cellcolor{cyan!30}65.54&\cellcolor{cyan!30} \textbf{494.92}&\cellcolor{cyan!30}\textbf{518.44}&\cellcolor{cyan!30}296.93&\cellcolor{cyan!30}390.77&\cellcolor{cyan!30}3m33s&\cellcolor{cyan!30}6m34s\\
&\cellcolor{cyan!30}$\text{\textbf{\ptome}}^{\text{BLIP-2}}_{r=0.975}$ &\cellcolor{cyan!30}\underline{\textbf{84.12}}&\cellcolor{cyan!30}\underline{\textbf{67.37}}&\cellcolor{cyan!30}\underline{\textbf{504.95}} &\cellcolor{cyan!30}\underline{\textbf{527.06}}&\cellcolor{cyan!30}\underline{\textbf{434.50}}&\cellcolor{cyan!30}\underline{\textbf{564.78}}&\cellcolor{cyan!30}\underline{\textbf{5m13s}}&\cellcolor{cyan!30}\underline{\textbf{9m24s}}\\

&\text{\color{gray!70}BLIP-2}&\text{\color{gray!70}85.32}&\text{\color{gray!70}68.26}& \text{\color{gray!70}507.46}& \text{\color{gray!70}528.63}&\text{\color{gray!70}678.45}&\text{\color{gray!70}900.77}&\text{\color{gray!70}7m52s}&\text{\color{gray!70}16m12s}\\
\bottomrule
\end{tabular}}
\label{tab:sota-img-text}
\end{minipage}
\hspace{0.05in}
\begin{minipage}[ht]{0.36\linewidth}
    \caption{\scolorbox{gray!15}{\textbf{Retrained}} \textbf{Image-Text Retrieval comparison} when retraining from scratch on CLIP and BLIP backbones.  $Rk = Rk@1 + Rk@5 + Rk@10,\,k\in\{t,i\}$, 
    }
 \resizebox{\textwidth}{!}{
 \setlength{\tabcolsep}{2.3pt}
\begin{tabular}{l|l|ccccc}
\toprule
 \textbf{Models} &\textbf{Algo.} &$\mathbf{Rt} \uparrow$&$\mathbf{Ri} \uparrow$&\textbf{GFLOPS} $\uparrow$&\textbf{\makecell{Eval\\Speed}}$\uparrow$&\textbf{\makecell{Train\\Speed}}$\uparrow$\\
\midrule
\multirow{6}{*}{$\text{CLIP}_{\text{Flickr}}$}
&Baseline& \text{\color{gray!70}291.80} & \text{\color{gray!70}275.52} & \text{\color{gray!70}x1.00} & \text{\color{gray!70}x1.00} & \text{\color{gray!70}x1.00}\\
 & ToMe & 287.30 & 270.52 & x2.10 & x1.39 &  {\textbf{x1.79}} \\
 & ToFu &288.32 & 269.68 & x2.10 & x1.39 & x1.76  \\
 & DCT &279.70 & 258.24 & x2.10 & x1.30 &  x1.37 \\
 & DiffRate&289.33& 266.45& x2.10 & x1.39& x1.78 \\
 & \cellcolor{gray!15}\textbf{\ptome} & \cellcolor{gray!15} \textbf{291.50} & \cellcolor{gray!15} \textbf{270.94} & \cellcolor{gray!15}\textbf{x2.10} & \cellcolor{gray!15}\textbf{x1.39} & \cellcolor{gray!15}\underline{x1.78}  \\
\midrule
\multirow{6}{*}{$\text{BLIP}_{\text{FLickr}}$}
&Baseline&\text{\color{gray!70}296.70} & \text{\color{gray!70}284.06} & \text{\color{gray!70}x1.00} & \text{\color{gray!70}x1.00}& \text{\color{gray!70}x1.00} \\
 & ToMe & 294.80 & 280.64 & x1.57& x1.66& \textbf{x1.60} \\
 & ToFu &{\textbf{296.46}} & 281.04 & x1.57& x1.65& x1.59  \\
 & DCT & 291.79 & 275.22 & x1.57& x1.61& x1.45 \\
 & DiffRate&292.77& 279.46& x1.57& x1.65& x1.59  \\
 & \cellcolor{gray!15}\textbf{\ptome} & \cellcolor{gray!15}\underline{296.00} & \cellcolor{gray!15}\textbf{282.36} & \cellcolor{gray!15}\textbf{x1.57} & \cellcolor{gray!15}\textbf{x1.66} & \cellcolor{gray!15}\underline{x1.59}  \\
 \midrule
\multirow{6}{*}{$\text{CLIP}_{\text{coco}}$}
&Baseline& \text{\color{gray!70}256.30} & \text{\color{gray!70}222.21} & \text{\color{gray!70}x1.00} & \text{\color{gray!70}x1.00} & \text{\color{gray!70}x1.00}  \\
 & ToMe &248.64 & 215.03 & x2.10 & x1.38 & x1.79  \\
 & ToFu & 248.99 & 216.56 & x2.10& x1.39 & x1.79\\
 & DCT & 240.04& 211.28& x2.10& x1.34& x1.37\\
 & DiffRate& 248.87 & 215.45& x2.10 & x1.39& x1.79  \\
 & \cellcolor{gray!15}\textbf{\ptome} & \cellcolor{gray!15}\textbf{250.70} & \cellcolor{gray!15}\textbf{217.01}& \cellcolor{gray!15} \textbf{x2.10}& \cellcolor{gray!15} \textbf{x1.39}& \cellcolor{gray!15} \textbf{x1.79}\\
\midrule
\multirow{6}{*}{$\text{BLIP}_{\text{coco}}$}
&Baseline& \text{\color{gray!70}273.72} & \text{\color{gray!70}241.30} & \text{\color{gray!70}x1.00} & \text{\color{gray!70}x1.00} & \text{\color{gray!70}x1.00}  \\
 & ToMe & 266.86 & 234.67& x1.57 & \textbf{x1.90}& x1.85  \\
 & ToFu  & 266.18 & 233.87& x1.57 & {x1.90}& x1.85\\
 & DCT & 264.38 & 230.19& x1.57 & x1.86& x1.78  \\
 & DiffRate& 265.45&  235.11 & x1.57& x1.84 & x1.85  \\
 & \cellcolor{gray!15}\textbf{\ptome} & \cellcolor{gray!15} \textbf{268.42} & \cellcolor{gray!15} \textbf{236.25}& \cellcolor{gray!15} \textbf{x1.57}& \cellcolor{gray!15} \underline{x1.88}& \cellcolor{gray!15} \textbf{x1.85}\\
 \bottomrule
\end{tabular}}
\label{tab:re-trained-img-text}
\end{minipage}
\end{table}
\vspace{-0.18in}
Let $L$ denote the number of layers in the CLIP encoder and $N$ the number of visually encoded tokens. In our experiment, we apply \ptome to the ViT vision encoder of LLAVA, retaining only $r$ percent of tokens in each layer. This results in $r^LN$ tokens being fed into the LLM, significantly enhancing inference speed. We used LLaVA-1.5-7B and LLaVA-1.5-13B checkpoints to run off-the-shelf settings. Tables \ref{tab:off-the-shelf-vqa} and \ref{tab:run_time_vqa}, along with Figure \ref{fig:ots-llava}, illustrate that the \ptome algorithm consistently achieves superior performance compared to other merging and pruning methods, as well as existing SOTA models
such as BLIP-2~\citep{BLIP2}, InstructBLIP~\citep{dai2024instructblip}, IDEFICS-9B/80B~\citep{idefics2023}, with \textit{inference time nearly halved}. Remarkably, in some datasets like VisWiz and ScienceQA, the compressed model even surpasses the baseline model. We contend that this improvement stems from the merging of less significant tokens in \ptome, potentially enhancing the robustness of the language model (LLM).
\vspace{-0.1in}
\subsection{Image Classification on Imagenet-1k}
\label{subsec:ic}
\vspace{-0.05in}
In this task, we employed five ViT backbones of varying sizes—tiny (ViT-T), small (ViT-S), base (ViT-B), large (ViT-L), and huge (ViT-H) - which are pre-trained using either MAE \citep{mae} or DEIT \citep{DEIT} styles. These backbones were utilized to assess both off-the-shelf and retrained performance. All experiments were conducted on the ImageNet-1k dataset, which is a subset of ImageNet \citep{imagenet} containing labeled images spanning 1000 categories. 
\vspace{-0.28in}
\begin{table}[H]
\centering
\begin{minipage}[ht]{.62\linewidth}
\vspace{-0.15in}
\caption{\textbf{\scolorbox{cyan!30}{{Off-the-shelf}}  \textit{LLaVA-1.5 7B} ($r$=0.9) and \textit{LLaVA-1.5 13B} ($r$=0.925)} performance vs. \ptome and other token pruning/merging methods on six VQA datasets:  VQA-v2 \citep{vqa_v2}, GQA \citep{gqa}, VisWiz \citep{vizwiz},  TextVQA \citep{textvqa}, MME \citep{mme}
ScienceQA image (ScienceQA\textsuperscript{I}) \citep{scienceqa}.}
\resizebox{\textwidth}{!}{
\label{tab:off-the-shelf-vqa}
\setlength{\tabcolsep}{1.8pt}
\scalebox{0.5}{
\setlength{\tabcolsep}{1.8pt}
\centering
\begin{tabular}{l|l|cccccccc}
\toprule
\textbf{Model}&\textbf{LLM}&\textbf{$\text{VQA}^{\text{v2}}$}$\uparrow$&\textbf{GQA} $\uparrow$& \textbf{VisWiz}$\uparrow$ & \textbf{ScienceQA\textsuperscript{I}}$\uparrow$ & \textbf{TextVQA} $\uparrow$& \textbf{MME}$\uparrow$\\
 \midrule
BLIP-2&Vicuna-13B&41.0&41.0&19.6&61.0&42.5&1293.8 \\
InstructBLIP&Vicuna-7B&-& 49.2&34.5&60.5& 50.1&-\\
InstructBLIP&Vicuna-13B&-&49.5&33.4& 63.1& 50.7&1212.8\\
IDEFICS-9B&LLaMA-7B&50.9&38.4&35.5&-&25.9&-\\
IDEFICS-80B&LLaMA-65B&60.0&45.2&36.0&-& 30.9&-\\
 \midrule
LLaVA-1.5-7B&\multirow{6}{*}{Vicuna-7B}&\text{\color{gray!70}76.6}&\text{\color{gray!70}62.0}&\text{\color{gray!70}54.4}&\text{\color{gray!70}70.4}&\text{\color{gray!70}46.0}&\text{\color{gray!70}1514.7}\\
\hspace{0.05in}\  ToMe &&75.2&59.5&55.9&68.7&41.1&1412.4\\
\hspace{0.05in}\  ToFu&&75.1&59.4&55.8&68.5&41.2&1405.3\\
\hspace{0.05in}\  DCT&&67.8&56.2&55.7&65.8&26.3&1193.9\\
\hspace{0.05in}\ DiffRate&&72.0&57.9&55.4&66.4&30.6&1341.0\\
\hspace{0.05in}\cellcolor{cyan!30}\textbf{\  \ptome}&\cellcolor{cyan!30}&\cellcolor{cyan!30}\textbf{75.4}&\cellcolor{cyan!30}\textbf{59.9}&\cellcolor{cyan!30}\textbf{55.9}&\cellcolor{cyan!30}\textbf{69.0}&\cellcolor{cyan!30}\textbf{43.0}&\cellcolor{cyan!30}\textbf{1448.1}\\
 \midrule
LLaVA-1.5-13B&\multirow{6}{*}{Vicuna-13B}&\text{\color{gray!70}78.3}&\text{\color{gray!70}63.2}&\text{\color{gray!70}56.7}&\text{\color{gray!70}72.8}&\text{\color{gray!70}48.7}&\text{\color{gray!70}1522.6
}\\
\hspace{0.05in}\  ToMe&&76.0&59.9&55.9&73.8&43.1&1470.3\\
\hspace{0.05in}\  ToFu&&76.1&60.1&56.1&74.0&43.0&1471.0\\
\hspace{0.05in}\  DCT&&70.8&57.3&56.1&70.3&23.9&1355.8\\
\hspace{0.05in}\  DiffRate&&73.4&58.5&54.6&70.6&32.8&1395.4\\
\hspace{0.05in}\cellcolor{cyan!30}\textbf{\  \ptome}&\cellcolor{cyan!30}&\cellcolor{cyan!30}\textbf{76.8}&\cellcolor{cyan!30}\textbf{60.2}&\cellcolor{cyan!30}\textbf{56.1}&\cellcolor{cyan!30}\textbf{74.0}&\cellcolor{cyan!30}\textbf{45.6}&\cellcolor{cyan!30}\textbf{1490.1}\\
\bottomrule
\end{tabular}}}
\vspace{0.1in}
\caption{\textbf{Inference time} of LLaVA-1.5-7B and LLaVA-1.5-13B models when running on \textit{five V100-GPUs} and \textit{five A100-GPUs}.}
\resizebox{0.98\textwidth}{!}{
\label{tab:run_time_vqa}
\centering
\hspace{0.1in}
\begin{tabular}{lllllll}
\toprule
\textbf{Model}& \textbf{VQA\textsuperscript{v2}} $\downarrow$ & \textbf{GQA} $\downarrow$ & \textbf{VisWiz} $\downarrow$ & \textbf{ScienceQA\textsuperscript{I}}$\downarrow$  & \textbf{TextVQA}$\downarrow$ & \textbf{MME}$\downarrow$ \\
\midrule
LLava-1.5-7B & 09h:05m & 10m:25s  & 04m:36s  & 01m:50s & 10m:12s  & 02m:32s \\
\  ToMe  & 05h:38m & 06m:34s  & 03m:26s  & 01m:07s & 07m:37s  & 01m:24s \\
\  ToFu  & 05h:35m & 06m:32s  & 03m:29s  & 01m:06s & 07m:40s  & 01m:24s \\
\  DCT   & 05h:59m & 06m:41s  & 03m:28s  & 01m:08s & 08m:16s  & 01m:27s \\
\  DiffRate& 05h:39m & 06m:39s  & 03m:26s  & 01m:06s & 07m:36s  & 01m:21s \\
\  \textbf{\ptome}& 05h:44m & 06m:37s  & 03m:26s  & 01m:07s & 07m:37s  & 01m:23s \\
\midrule
LLava-1.5-13B & 13h:11m & 13m:05s & 07m:36s&04m:54s& 15m:04s &02m:59s \\
\  ToMe & 09h:28m& 09m:35s & 05m:58s & 03m:31s &11m:48s& 02m:16s\\
\  ToFu & 09h:26m& 09m32s& 05m:58s& 03m:26s& 11m:45s& 02m:15s\\
\  DCT & 10h:02m & 10m:53s& 06m:46s& 03m:45s& 12m:57s & 02m:34s\\
\  DiffRate& 09h:33m& 09:m44s& 06m:01s& 03m:37s & 11m:52s& 02m:18s\\
\  \textbf{\ptome}& 09h:32m& 09m:39s& 06m:03s & 03m:35s & 12m:08s& 02m:17s\\
\bottomrule
\end{tabular}}

\end{minipage}
\hspace{0.05in}
\begin{minipage}[ht]{.3\linewidth}
    \centering
    \includegraphics[width=1.0\textwidth]{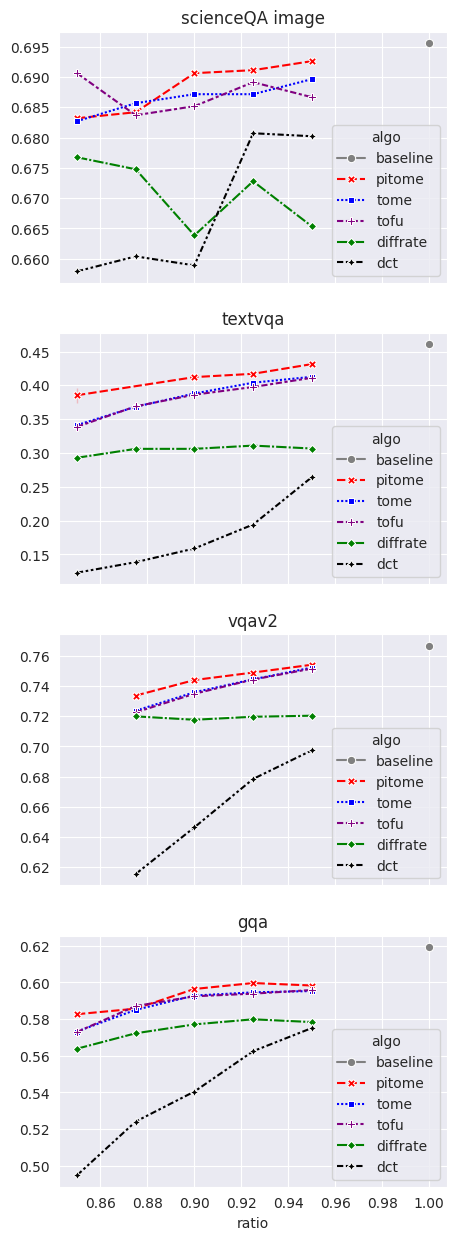}
    \captionof{figure}{\textbf{\scolorbox{cyan!30}{Off-the-shelf}} performance of \ptome on \textbf{LLaVA-1.5-7B} with different compressing ratio $r$.}
      \label{fig:ots-llava}
\end{minipage}
\end{table}
\vspace{-0.28in}
Table \ref{tab:ic} and Figure \ref{fig:mae-deit} present our experimental results, comparing \ptome with recent works, including SOTA efficient transformers such as Swin-B~\citep{CSWin-B}, CSWin-B~\citep{CSWin-B}, MViT-B/L~\citep{MViTv2-B}, MAE~\citep{mae},
and other token merging/pruning methods \citep{dynamic-vit,a-vit,spvit}. We observe that \ptome 
maintains high accuracy with an average performance drop of only $0.5\%$ after reducing up to $44\%$ of FLOPS (MAE-H), showcasing superior performance with comparable throughput. It is important to note that \textit{dynamic pruning-based methods} such as A-ViT~\citep{a-vit}, Dynamic ViT~\citep{dynamic-vit}, and SP-ViT~\citep{kong2022spvit} do not accelerate training speed due to using additional masks for padding tokens into a same dimension. On the retraining settings, we note that models compressed by \ptome also surpass merging/pruning methods by a large margin and approach the performance of the original models.

\vspace{-0.1in}
\subsection{Text Classification}\label{sec:text-class}
\vspace{-0.1in}
While previous studies have focused on benchmarking BSM-based algorithms within the vision or vision-language domain, we also extend experiments to the text domain, where input sequence lengths vary by sample. Specifically, we apply compression algorithms to the first three layers of the BERT model \citep{devlin2018bert}, reducing the number of tokens by $20\%$ in each layer. Our experiments utilize the SST-2 dataset ~\citep{sst2} with an average sequence length of 23.2 tokens and the IMDb dataset ~\citep{imdb} with an average sequence length of 292.2 tokens.

As demonstrated in Table \ref{table:text-performance} and Figure \ref{fig:text_ots} (Appendix), our findings indicate that \ptome performs better than other BSM-based baselines. Additionally, after retraining, the compressed BERT models achieve competitive records while significantly accelerating training speed compared to previous pruning methods such as PowER-BERT \citep{power-bert}, Fisher \citep{Fisher}, and LTP \citep{ltp}, as well as BERT-based efficient models like DistilBERT \citep{distilbert} and ALBERT \citep{albert}. Notably, we observe only a $0.4\%$ performance drop on the IMDb dataset and even surpass the original BERT model by $0.3\%$ on the SST-2 dataset. For detailed empirical results on this task, please refer to Appendix \ref{sec:text-appendix}.
\vspace{-0.1in}
\begin{table}[H]
\centering
\begin{minipage}[ht]{0.38\linewidth}
\captionof{table}{
    \textbf{Image Classification}: Performance of \ptome on Imagenet-1k, both \hlc[cyan!30]{off-the-shelf} (OTS acc) and after \hlc[gray!15]{retraining}(Trained acc), across ViT backbones. We benchmark with different architectures and merging/pruning methods.
}\label{tab:ic}
\centering
 \resizebox{1.0\textwidth}{!}{
 \setlength{\tabcolsep}{2.2pt}
    \begin{tabular}{ll|>{\columncolor{cyan!30}}c>{\columncolor{gray!15}}ccc}
    \toprule
    \textbf{Type} & \textbf{Model} & \cellcolor{white}\textbf{\makecell{OTS \\ Acc.}} & \cellcolor{white}\textbf{\makecell{Trained\\Acc.}}& \textbf{\makecell{Flops $\downarrow$}}& \textbf{\makecell{Train\\speed up}}\\
    \midrule
   \multirow{4}{*}{\makecell{Other\\models}}&Swin-B&n/a&84.0&15.4&$\times$\\
    &CSWin-B&n/a&84.2&15.0&$\times$\\
    &MViTv2-B&n/a&84.4&10.2&$\times$\\
    &MViTv2-L&n/a&85.3&42.1&$\times$\\
    \midrule
    \multirow{5}{*}{merge}
    &$\text{ToMe}^{\text{DEIT-T}}$&68.9&70.0&0.79&\checkmark\\
    &$\text{ToFu}^{\text{DEIT-T}}$&69.6&70.5&0.79&\checkmark\\
    &$\text{DCT}^{\text{DEIT-T}}$&67.6&68.7&0.79&\checkmark\\
    &$\text{DiffRate}^{\text{DEIT-T}}$&69.9&70.7&0.79&\checkmark\\
    &$\text{\textbf{\ptome}}^{\text{DEIT-T}}$&\textbf{70.8}&\textbf{71.6}&\textbf{0.79}&\checkmark\\
     &\text{\color{gray!70}$\text{ViT}^{\text{DEIT-T}}$}&\cellcolor{white}\text{\color{gray!70}72.3}&\cellcolor{white}\text{\color{gray!70}72.3}&\text{\color{gray!70}1.2}&$\times$\\
         \midrule 
    \multirow{3}{*}{prune} 
    &$\text{A-ViT}^{\text{DEIT-S}}$&n/a&78.6&2.9&$\times$\\
    &$\text{Dynamic-ViT}^{\text{DEIT-S}}$&n/a&79.3&2.9&$\times$\\
    &$\text{SP-ViT}^{\text{DEIT-S}}$&n/a&79.3&2.6&$\times$\\
    \cline{2-6}
    \multirow{6}{*}{merge} 
    &$\text{E-ViT}^{\text{DEIT-S}}$&-&79.5&2.9&$\times$\\
    &$\text{ToMe}^{\text{DEIT-S}}$&77.7&79.4&2.9&\checkmark\\
    &$\text{ToFu}^{\text{DEIT-S}}$&77.8&79.6&2.9&\checkmark\\
    &$\text{DCT}^{\text{DEIT-S}}$&74.8&78.6&2.9&\checkmark\\
    &$\text{DiffRate}^{\text{DEIT-S}}$&76.8&79.5&2.9&\checkmark\\
    &$\text{\textbf{\ptome}}^{\text{DEIT-S}}$&\textbf{79.1}&\textbf{79.8}&\textbf{2.9}&\checkmark\\
    &\text{\color{gray!70}$\text{ViT}^{\text{DEIT-S}}$}&\cellcolor{white}\text{\color{gray!70}79.8}&\cellcolor{white}\text{\color{gray!70}79.8}&\text{\color{gray!70}4.6}&$\times$\\  \midrule 
    \multirow{5}{*}{merge} &
    $\text{ToMe}^{\text{MAE-L}}$&82.9&85.0&31.0&\checkmark\\ 
    &$\text{ToFu}^{\text{MAE-L}}$&83.8&85.1&31.0&\checkmark\\
    &$\text{DCT}^{\text{MAE-L}}$&82.8&84.4&31.0&\checkmark\\
    &$\text{DiffRate}^{\text{MAE-L}}$&83.2&85.3&31.0&\checkmark\\
    &$\text{\textbf{\ptome}}^{\text{MAE-L}}$&\textbf{84.6}&\textbf{85.3}&\textbf{31.0}&\checkmark\\
     & \text{\color{gray!70} $\text{ViT}^{\text{MAE-L}}$}&\cellcolor{white}\text{\color{gray!70}85.7}&\cellcolor{white}\text{\color{gray!70}85.7}&\text{\color{gray!70}61.6}&$\times$\\ 
    \cline{2-6}
    \multirow{5}{*}{merge}
    &$\text{ToMe}^{\text{MAE-H}}$&85.6&86.4&92.8&\checkmark\\
    &$\text{ToFu}^{\text{MAE-H}}$&85.8&86.4&92.8&\checkmark\\
    &$\text{DCT}^{\text{MAE-H}}$&84.3&86.0&92.8&\checkmark\\
    &$\text{DiffRate}^{\text{MAE-H}}$&85.9&86.6&92.8&\checkmark\\
    &$\text{\textbf{\ptome}}^{\text{MAE-H}}$&\underline{\textbf{86.4}}&\underline{\textbf{86.7}}&\textbf{92.8}&\checkmark\\
     &\text{\color{gray!70}$\text{ViT}^{\text{MAE-H}}$}&\cellcolor{white}\text{\color{gray!70}86.9}&\cellcolor{white}\text{\color{gray!70}86.9}&\text{\color{gray!70}167.4}&$\times$ \\
     \bottomrule
  \end{tabular} 
}
  \end{minipage}
  \hspace{0.02in}
  \begin{minipage}[ht]{.6\linewidth}
	 \centering
   	\captionof{table}{\textbf{Text Classification}: \ptome vs other BERT-style compressed models and token pruning ones.} \label{table:text-performance}
   		\resizebox{0.75\textwidth}{!}{
   			\begin{tabular}{lll|ccc}
   				\toprule
   				\textbf{Dataset} & \textbf{Type} & \textbf{Model} &\textbf{Acc}  & \textbf{\makecell{Eval\\Flops $\uparrow$}} & \textbf{\makecell{Train\\Speed $\uparrow$}}\\
     			\midrule
     			\multirow{12}{*}{SST-2} &\multirow{2}{*}{\makecell{compressed \\ models}}
     			&ALBERT&91.3 &x1.0&x1.1\\
     			&&DistiledBERT& 91.1 & x2.0 &  x1.7 \\  [1mm]
     			\cline{2-6}
     			&&\text{\color{gray!70}BERT}& \text{\color{gray!70}91.4} & \textbf{\color{gray!70}x1.0} &  \text{\color{gray!70}x1.0} \\ 
   				&\multirow{3}{*}{\makecell{pruning\\+mask}}&PowER-BERT& 91.1& x2.5 & x1.0 \\ 
   				&&Fisher& 91.3& x1.6 &  x1.0 \\
   				&&LTP& 91.3& x2.9 &  x1.0 \\
   				\cline{2-6}
   				&\multirow{5}{*}{\makecell{merging}}&\cellcolor{cyan!30} \ptome& \cellcolor{cyan!30}{91.0}& \cellcolor{cyan!30}{x1.9}& \cellcolor{cyan!30}{x1.4}\\
   				&&\cellcolor{gray!15} ToMe& \cellcolor{gray!15}91.2& \cellcolor{gray!15}{x1.9}& \cellcolor{gray!15}x1.4\\
   				&&\cellcolor{gray!15} ToFu& \cellcolor{gray!15}89.8& \cellcolor{gray!15}{x1.9}& \cellcolor{gray!15}x1.4\\
   				&&\cellcolor{gray!15} DCT& \cellcolor{gray!15}90.7& \cellcolor{gray!15}{x1.9}& \cellcolor{gray!15}x1.4\\
   				&&\cellcolor{gray!15} DiffRate& \cellcolor{gray!15}89.7& \cellcolor{gray!15}{x1.9}& \cellcolor{gray!15}x1.4\\
   				&&\cellcolor{gray!15} \ptome& \cellcolor{gray!15}\textbf{91.7}& \cellcolor{gray!15}\textbf{x1.9}& \cellcolor{gray!15}\textbf{x1.4}\\
     			\midrule
     			\multirow{11}{*}{IMDb} &\multirow{2}{*}{\makecell{compressed\\models}}
     			&ALBERT&89.2 &x1.0&x1.2\\
   				&&DistiledBERT& 93.0 & x2.0 &  x1.9 \\ [1mm]
   				\cline{2-6}
   				&&\text{\color{gray!70}BERT}& \text{\color{gray!70}94.0} & \text{\color{gray!70}x1.0} &  \text{\color{gray!70}x1.0} \\ 
   				&\multirow{2}{*}{\makecell{pruning\\+mask}}& PowER-BERT& 92.5  & x2.7 & x1.0 \\ 
   				&& TR-BERT& 93.6  & x2.3 & x1.0 \\ 
   				\cline{2-6}
   				&\multirow{3}{*}{\makecell{merging}}&\cellcolor{cyan!30} \ptome& \cellcolor{cyan!30}{93.2}& \cellcolor{cyan!30}{x1.9} & \cellcolor{cyan!30}{x1.8} \\
   				&&\cellcolor{gray!15} ToMe& \cellcolor{gray!15}93.3& \cellcolor{gray!15}{x1.9} & \cellcolor{gray!15}{x1.8} \\
   				&&\cellcolor{gray!15} ToFu& \cellcolor{gray!15}92.6& \cellcolor{gray!15}{x1.9} & \cellcolor{gray!15}{x1.8} \\
   				&&\cellcolor{gray!15} DCT& \cellcolor{gray!15}92.4& \cellcolor{gray!15}{x1.9} & \cellcolor{gray!15}{x1.8} \\
   				&&\cellcolor{gray!15} DiffRate& \cellcolor{gray!15}92.4& \cellcolor{gray!15}{x1.9} & \cellcolor{gray!15}{x1.8} \\
   				&&\cellcolor{gray!15} \ptome& \cellcolor{gray!15}\underline{\textbf{93.6}}& \cellcolor{gray!15}\textbf{x1.9} & \cellcolor{gray!15}\textbf{x1.8} \\
   				\bottomrule
     		\end{tabular}}
       \vspace{0.15in}
          \hspace{-0.15in} \includegraphics[width=1.0\textwidth]{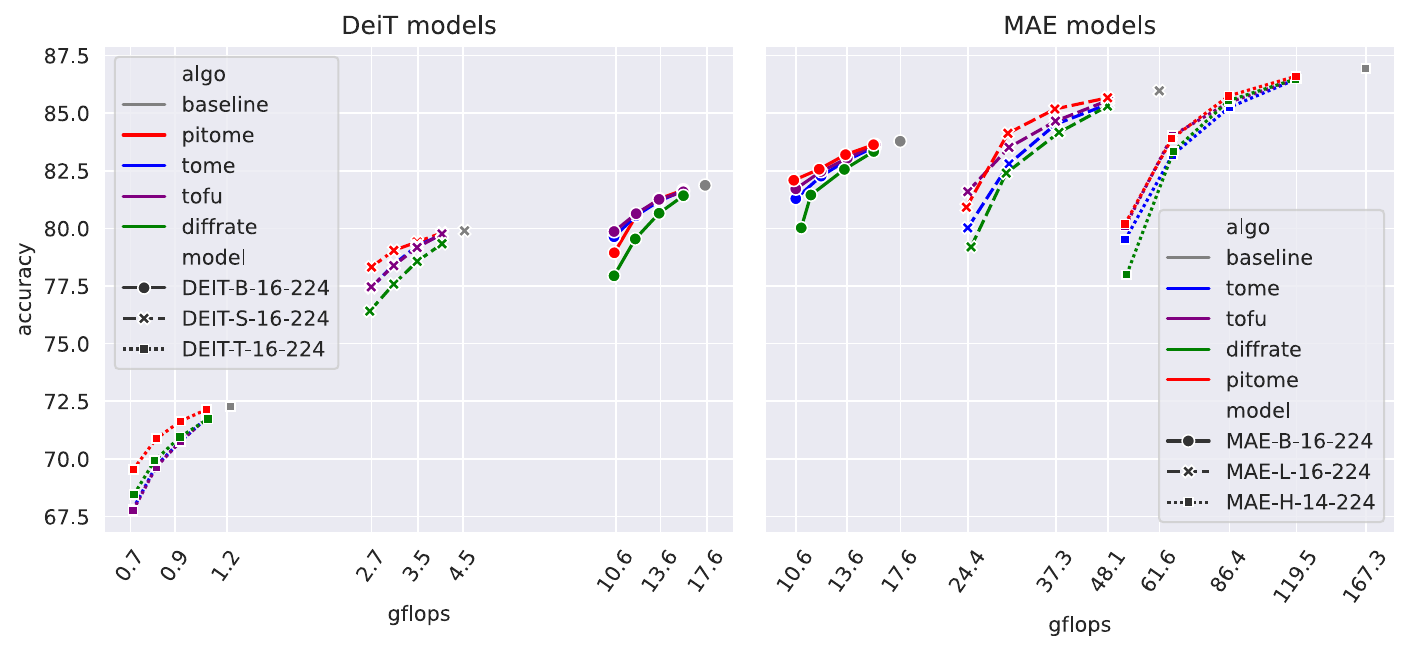}
          \vspace{-0.1in}
    \captionof{figure}{\textbf{\scolorbox{cyan!30}{Off-the-shelf}  results on Imagenet-1k. Zoom in for better view.}}
    \label{fig:mae-deit}
  \end{minipage}
\end{table}
\vspace{-0.3in}
\subsection{\ptome Ablation Studies}\label{sec:schedule}
\vspace{-0.1in}
\textbf{Contributions of energy scores and related factors.}
To assess the performance of the components used in \ptome, we conduct the following settings: 
(i) \ptome without protecting important tokens by our energy in Step 2, i.e., using odd and even indices in sorted energy score array as two sets in BSM; (ii) \ptome where the merging process in Step 3 conducted on two randomly sets $\mathcal{A},\,\mathcal{B}$ as baselines~\citep{tome,kim2024token} instead of leveraging ordered in sorted energy vectors $\mathbf{E(.)}$; (iii) \ptome without using our proposed energy score as in Eq\eqref{equa:energy} but utilizing other indicators like attention scores from the $\mathrm{[CLS]}$ (\ptome w cls attn) token~\citep{DiffRate} or mean of attention scores; (iv) \ptome 
using a fixed of $k$ removed token at each layer as ToMe~\citep{tome}
rather than a reducing ratio of $r$ as our configuration.

We run experiments on image-text retrieval and text classification tasks, reporting the results in Table~\ref{tab:ablation} for (i) and (ii), and in Figure~\ref{fig:ablation} for (iii) and (iv). The results demonstrate that all factors contribute to the performance of \ptome, with energy-based operations playing a particularly significant role. Additionally, reducing tokens with a ratio $r$ effectively eliminates redundant tokens in early layers while preserving informative ones in later layers. 

\textbf{Margin $m$ and $\alpha$ hyper-parameters.}
To validate the roles of these parameters in our energy score function in Eq.\eqref{equa:energy}, we conduct ablation studies on image-text retrieval task with (v) adaptive margin $m$ compared with a fixed value $m \in \{0.9, 0.45, 0.0, -1.0\}$ when varying the ratio $r$ and (vi) given a fixed value of $r$, changing the smooth constant value $\alpha$ in $\alpha(\exp(x-m) - 1)$ with $x < m$. Results for these settings are summarized in Figure \ref{fig:ablation_m} and Table \ref{tab:ablation_alpha}, respectively. We observe that while models with fixed tend to have the accuracy drop sharply when it is lower than some threshold, the adaptive margins achieve the best results across cases. We hypothesize that as the token space becomes sparser in deeper layers, PiToMe’s fixed $m$ approach likely assigns the same energy score to all tokens, making it difficult to isolate and protect tokens during merging. Table \ref{tab:ablation_alpha} also shows that $\alpha = 1.0$ is the best choice across margin values. 

Further details, including additional ablation study results, visualizations (output merging, open-chat with LLaVa), and extra \ptome experiments, are provided in the Appendix.

\begin{minipage}{0.55\textwidth}
\captionof{table} {\textbf{Impact} of different settings in \textbf{Steps 2 and 3}.}\label{tab:ablation}
\resizebox{0.9\textwidth}{!}{
\begin{tabular}{l|
lc|lc}
\toprule
\multirow{2}{*}{\textbf{Settings}}&\multicolumn{2}{|c}{\textbf{Image-Text Retrieval}}&\multicolumn{2}{|c}{\textbf{Text CLS.}}\\ 
\cline{2-5}
&ratio $r$& $Rsum$ &ratio $r$&acc\\
\midrule
  \multirow{3}{*}{\makecell{\ptome w/o\\ protecting tokens \\in  step 2}}&0.925&568.9&0.6&86.99\\
  &0.95&575.3&0.7&89.97\\ 
  &0.975&578.2&0.8&91.67\\
  \midrule
  \multirow{3}{*}{\makecell{\ptome using \\random split \\ in step 3}}&0.925&567.9&0.6&87.17\\ 
  &0.95&574.7&0.7&90.34\\
  &0.975&578.4&0.8&91.73\\
  \midrule
  \multirow{3}{*}{\ptome}& \cellcolor{cyan!30}0.925& \cellcolor{cyan!30}\textbf{573.4}& \cellcolor{cyan!30}0.6& \cellcolor{cyan!30}\textbf{89.20}\\ 
  & \cellcolor{cyan!30}0.95& \cellcolor{cyan!30}\textbf{577.8}& \cellcolor{cyan!30}0.7& \cellcolor{cyan!30}\textbf{91.47}\\
  & \cellcolor{cyan!30}0.975& \cellcolor{cyan!30}\textbf{580.1}& \cellcolor{cyan!30}0.8& \cellcolor{cyan!30}\textbf{93.26}\\
\bottomrule
\end{tabular}}
\end{minipage}
\hfill
\begin{minipage}{0.43\textwidth}
\vspace{-0.15in}
\captionof{table} {\textbf{Impact} of the constant $\alpha$ on the image-text retrieval task. Results are in recall sum; higher is better.}\label{tab:ablation_alpha}
\resizebox{0.9\textwidth}{!}{
\begin{tabular}{lccc}
\toprule
ration $r$     & $\alpha = 1.0$ & $\alpha = 0.5 $ & $\alpha = 0.$ \\
\midrule
0.85  & \textbf{519.98}         & 518.66          & 515.90         \\
0.875 & \textbf{545.90}          & 544.22          & 542.54        \\
0.90   & \textbf{562.82}         & 562.42          & 561.92        \\
0.925 & \textbf{571.88}         & 571.10           & 570.62   \\     
0.95  & \textbf{577.50}          & 577.43          & 577.40         \\
0.975 & \textbf{580.24}         & 579.82          & 579.76   \\
\bottomrule
\end{tabular}}
\end{minipage}
\vspace{-0.05in}
\begin{minipage}{{0.48\textwidth}}
        \centering
     \includegraphics[width=0.93\textwidth]{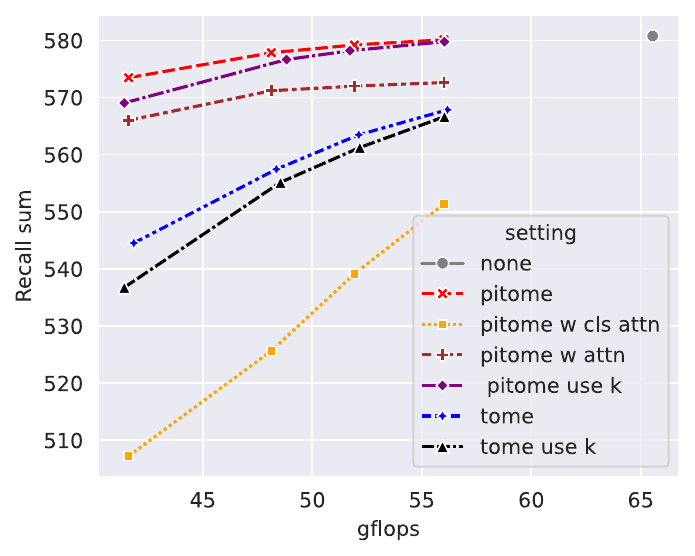}
     \vspace{-0.1in}
     \captionof{figure}{Ablation studies of \ptome.}
    \label{fig:ablation}
\end{minipage}
\hfill
\begin{minipage}{{0.44\textwidth}}
        \centering     
        \includegraphics[width=0.9\textwidth]{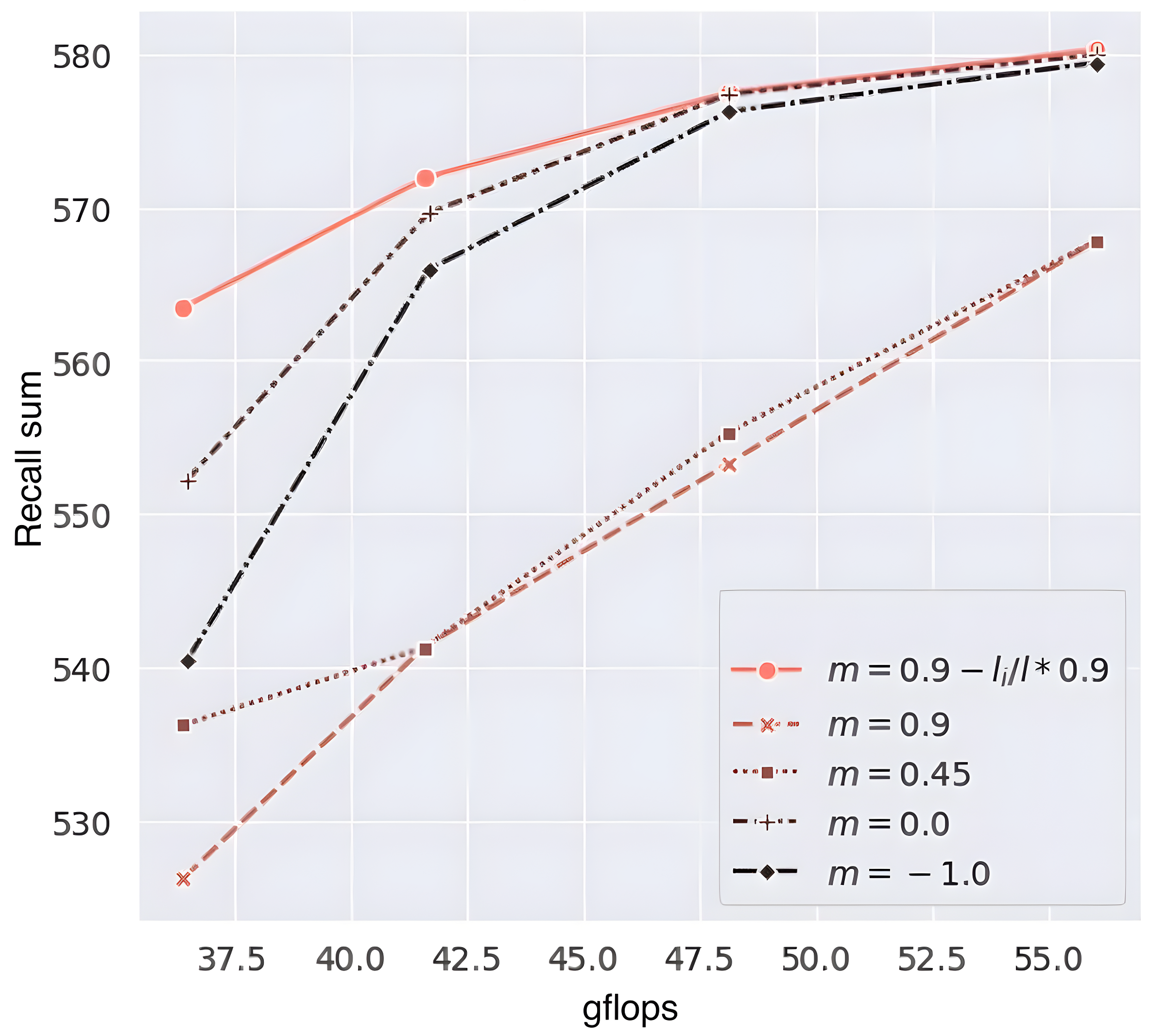}
     \vspace{-0.1in}
     \captionof{figure}{Ablation studies on adaptive margin $m$.}
    \label{fig:ablation_m}
\end{minipage}

\vspace{-0.1in}
\section{Conclusion}
\vspace{-0.15in}
This paper introduces \ptome, a novel algorithm that employs energy concepts to protect informative tokens during the token merging process. Our algorithm matches the efficiency of heuristic merging methods while maintaining a theoretical connection to the spectral properties of the input token space. In experiments on image classification, image-text retrieval, and VQA with LLaVA-1.5 7B/13B, \ptome consistently outperforms recent token merging and pruning methods, given the equivalent runtime and memory usage. 

\textbf{Limitations and Future Works}
Although our focus has been on tasks using ViT encoders for a variety of applications, we believe it is important to extend \ptome to generative tasks such as image generation (e.g., stable diffusion) or segmentation. This extension, however, necessitates the development of an \textit{unmerge mechanism} in the decoder, which remains an open question. Additionally, our energy score relies on a fully connected graph of input tokens, which can increase complexity as the input size grows. Constructing sparse graphs, therefore, might be beneficial for scaling in more challenging settings. Finally, designing a differentiable learning mechanism to optimize the reducing rate $r$ for token merging could enhance robustness and versatility across different downstream tasks.

\newpage
\subsection*{Acknowledgment}
The authors thank the International Max Planck Research
School for Intelligent Systems (IMPRS-IS) for supporting
Duy M. H. Nguyen. Duy M. H. Nguyen and Daniel Sonntag are also supported by the XAINES
project (BMBF, 01IW20005), No-IDLE project (BMBF,
01IW23002), and the Endowed Chair of Applied Artificial Intelligence, Oldenburg University.  Hoai-Chau Tran acknowledges the support from the AISIA Extensive Research Assistant Program 2023 (Batch 1) during this work and DFKI for supporting computing resources. TrungTin Nguyen acknowledges support from the Australian Research Council grant DP230100905. 
Ngan Le acknowledges funding support from the U.S. National Science Foundation (NSF) under Award No. OIA-1946391 and NSF EFRI BRAID 2223793. Binh T. Nguyen wants to thank the University of Science, Vietnam National University in Ho Chi Minh City for their support. Mathias Niepert acknowledges funding by Deutsche Forschungsgemeinschaft (DFG, German Research Foundation) under Germany’s Excellence Strategy - EXC and support by the Stuttgart Center for Simulation Science (SimTech). 


    

\addtocontents{toc}{\protect\setcounter{tocdepth}{2}}

\bibliography{references}

\appendix

\newpage
\begin{center}
{\bf \Large{Supplement to
``Accelerating Transformers with Spectrum-Preserving Token Merging''}}
\end{center}
\tableofcontents
\appendix
\begin{figure}[H]
    \centering
    \resizebox{0.8\textwidth}{!}{
    \input{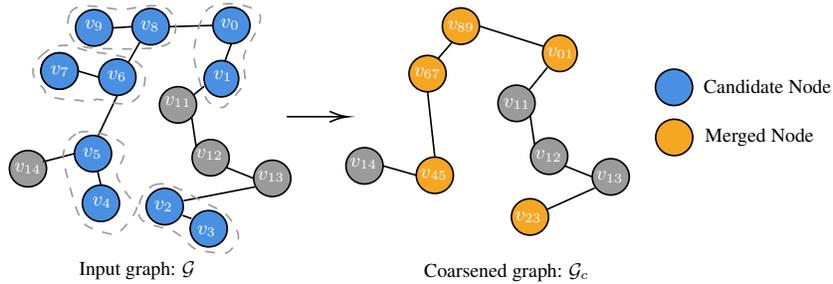}}
    \caption{Token merging outputs can be seen as coarsened graph from an input graph.}
    \label{fig:graph-coarsen}
\end{figure}

\section{Datasets Descriptions}\label{sec:datasets_details}
We present in Table \ref{tab:all_datasets} datasets used in our experiment. It is important to note that \ptome was run off-the-shelf on large datasets in VQA tasks such as VQA-v2 with $~447.8k$ sample or GQQ with $12.5$ questions, validating generalization and robustness of our merging mechanism.

\begin{table}[H]

\caption{Brief statistic of all datasets used in this paper.}\label{tab:all_datasets}
\centering
 \resizebox{\textwidth}{!}{
 \setlength{\tabcolsep}{1.8pt}
    \begin{tabular}{l|p{2cm}|p{2cm}|l|l}
    \toprule
    \textbf{Task} & \textbf{Name} & \textbf{Modality}& \textbf{No. Train}& \textbf{No. Test}\\
    \midrule
    \multirow{2}{*}{Image-Text retrieval}&Flickr30k \citep{FLickr30k} & Vision, Text &29k images, each with 5 captions & 1k images, each with 5 captions\\
    &MS-COCO\citep{MSCOCO}& Vision, Text &approx 118k images, each with 5 captions  & 5k images, each with 5 captions\\
    \midrule
    \multirow{6}{*}{\makecell{Visual Question \\Answering}}&VQA-v2\citep{vqa_v2} & Vision, Text & approx 443.7k questions& approx 447.8k questions\\
    &GQA \citep{gqa}& Vision, Text & approx 72k images with 943k questions&approx 10.2k images with 12.5k questions\\
    &VisWiz \citep{vizwiz}& Vision, Text & 8k image and questions & approx 4.32k image and questions\\
    &TextVQA\citep{textvqa}& Vision, Text & approx 28.4k images with 34.6k questions & approx 3.1k images with 4.2k questions\\
    &MME \citep{mme}& Vision, Text &- & 2375 images and questions\\
    &ScienceQA\citep{scienceqa}& Vision, Text & approx 4.11k image and questions & approx 2.11k image and questions\\
    \midrule
    \multirow{2}{*}{Text Classification}&IMDb \citep{imdb}& Text & 25k movie reviews.&25k movie reviews.\\
    &SST-2 \citep{sst2}& Text & 67,349 sentences & 872 sentences\\
    \midrule
    Image Classification&Imagenet-1k\citep{imagenet}& Vision&Approximately approx 1.28 million images &50k images (50 images per class)\\
    \midrule
  \end{tabular} 
}
 
\end{table}

\section{\ptome Algorithm}
\phantomsection
\subsection{Pseudo-Code Implementation}
The pseudo-code for our method is provided in Algorithm \ref{algo:pitome}. Here $\textbf{E}$ in line 3 is a vector that represents energy scores of all nodes calculated by the energy function \ref{equa:energy}. The final output is a $\mathrm{MERGE(.)}$ function, which serves as a lambda function that can be applied to any matrix $\mathbf{\hat{X}}^l$ at the $l$-th layer. The vector $\textbf{m}$ is a vector that contains information about token sizes (i.e, the number of tokens being merged into each token).
\vspace{-0.2in}
\begin{figure}[!ht]
    \centering
    \begin{minipage}[ht]{0.9\textwidth}
    \scalebox{1.0}{
        \begin{minipage}{\linewidth}
                    \begin{algorithm}[H]
            \caption{\ptome Algorithm}
            \label{algo:pitome}
            \begin{algorithmic}[1]
                \Function{pitome}{$\text{remain token ratio: } r, \text{input graph: } \mathcal{G}(\mathcal{V}, \mathcal{E})$} \Comment{Function to prepare for merging}
                    \State $k \gets N - N \cdot r$ \Comment{Compute number of nodes to merge}
                    \State $\textbf{s} \gets \text{argsort}(\bfE, \text{descending=True})$ \Comment{Compute energy scores}
                    \State $\mathrm{\textit{merge}},\,\mathrm{\textit{protect}} \gets \textbf{s}[:2 \cdot k], \textbf{s}[2 \cdot k:]$ \Comment{Identify mergeable and protected nodes}
                    \State $\textbf{n}_{a}, \textbf{n}_{b} \gets \mathrm{\textit{merge}}[::2], \mathrm{\textit{merge}}[1::2]$ \Comment{Split mergeable nodes}
                    \State $\mathcal{E}_{merge} \gets \mathcal{E}[\textbf{n}_a][\textbf{n}_b]$ \Comment{Get edge weights of mergeable nodes}
                    \State $\textbf{n}_{dst} \gets \text{argmax}(\mathcal{E}_{merge})$ \Comment{Find closest neighbors}
                    \Function{Merge}{$\textbf{X}$} \Comment{Function to perform merging}
                        \State $\textbf{X}_{protected} \gets \textbf{X}[\mathrm{\textit{protect}},:]$ \Comment{Extract protected tokens}
                        \State $\textbf{X}_{\mathcal{A}}, \textbf{X}_{\mathcal{B}} \gets \textbf{X}[\textbf{n}_{a}, :], \textbf{X}[\textbf{n}_{b}, :]$ \Comment{Extract tokens in set $\mathcal{A}$ and $\mathcal{B}$}
                        \State $\textbf{X}_{\mathcal{A}}, \textbf{X}_{\mathcal{B}} \gets \textbf{X}_{\mathcal{A}} \times \textbf{m}[\textbf{n}_{a}], \textbf{X}_{\mathcal{B}} \times \textbf{m}[\textbf{n}_{b}]$ \Comment{Weighted average}
                        \State $\textbf{X}_{\mathcal{B}} \gets \textbf{X}_{\mathcal{B}}.\text{scatter\_reduce}(\textbf{n}_{dst}, \textbf{X}_{\mathcal{A}}, \text{mode}=\text{"sum"})$ \Comment{Merge tokens}
                        \State $\textbf{X}_{\mathcal{B}} \gets \textbf{X}_{\mathcal{B}}/\textbf{m}_{\mathcal{B}}.\text{scatter\_reduce}(\textbf{n}_{dst}, \textbf{m}_{\mathcal{A}}, \text{mode}=\text{"sum"})$ \Comment{Weighted average}
                        \State \textbf{return} $\text{cat}(\textbf{X}_{protected}, \textbf{X}_{\mathcal{B}})$ \Comment{Concatenate and return merged tokens}
                    \EndFunction
                    \State \textbf{return} MERGE \Comment{Return merging lambda function}
                \EndFunction
            \end{algorithmic}
        \end{algorithm}
        \end{minipage}}
    \end{minipage}%
    \hspace{-0.1in}
    \begin{minipage}[ht]{0.0\textwidth}
    \end{minipage}
\end{figure}

\subsection{\ptome Complexity Analysis}
In algorithm \ref{algo:pitome}, in line $3$, the weighted graph is constructed through matrix multiplication, leading to a complexity of $\mathcal{O}\left(N^2h\right)$, where $h$ is the dimension of input vectors. Next, the computed energy scores are sorted, which have the complexity of $\mathcal{O}\left(N\log(N)\right)$. Lastly, in line $7$, the max operator for selecting the merge destination and the lambda $MERGE$ function, which performs tensor operations based on computed indices, lead to linear complexity. Combining these aspects, the overall time complexity of the \ptome function can be approximated as $\mathcal{O}\left(N^2h\right)$, considering the dominant factors contributing to computational cost. However, actual performance may vary depending on the specific PyTorch version and hardware utilization, with optimizations potentially altering these estimates. 

 In the BSM algorithms used by ToMe, after the bipartition of tokens into two sets $\mathcal{A}$ and $\mathcal{B}$ using odd and even indices, each set has $N/2$ tokens, it also requires calculating a similarity matrix between tokens in these two sets.   This operation also has the complexity of $\mathcal{O}\left(N^2h\right)$. The similarity matrix is then sorted to get the top $k$ tokens in set $\mathcal{A}$ with the highest similarity score used for merging, this operation also has the time complexity of $\mathcal{O}\left(N\log(N)\right)$, and the merging procedure is the same as our \ptome algorithm. So basically when compared to BSM, our algorithms have the same time complexity of $\mathcal{O}\left(N^2h\right)$, although in practice the speed of ToMe is a little bit faster than \ptome (a few milliseconds), but our algorithms give a much better trade-off between speed and accuracy. For more details about the BSM algorithms please refer to the ToMe paper\cite{tome}.
\subsection{Model complexity analysis}
In a standard Transformer layer, we have the time and space complexity to be $\mathcal{O}\left(N^2h + Nh^2\right)$ and $\mathcal{O}\left(N^2 + Nh\right)$. In each layer $ith$ of the model, we compress the number of tokens down to $rN$ using \ptome algorithms with the complexity of $\mathcal{O}\left((r^{i-1}N)^2h\right)$, so the layer $i$  will enjoy the time and space complexity of $\mathcal{O}\left((r^{i-1}N)^2h + r^iNh^2 + (r^{i-1}N)^2h\right)$ and $\mathcal{O}\left((r^{i-1}N)^2 + r^{i}Nh\right)$ \respn 
 Vision language models like LLaVA directly use output tokens from ViT encoders.  Let $l$ be the number of layers in the ViT encoder since we utilize \ptome in each layer in the ViT encoder part; the LLM model will have the time and space complexity of $\mathcal{O}\left((r^{l}N_{\mathrm{ViT}}+N_{\mathrm{LLM}})^2h + (r^{l}N_{\mathrm{ViT}}+N_{\mathrm{LLM}})h^2\right)$
and $\mathcal{O}\left((r^{l}N_{\mathrm{ViT}}+N_{\mathrm{LLM}})^2 + (r^{l}N_{\mathrm{ViT}}+N_{\mathrm{LLM}})h\right)$ which boosts the inference speed and saves a high amount of memory usage. The speedup could get even more impressive when used with higher batch size and image size. This also applies to models that use cross-attention modules for image text matching like BLIP and ALBEF, in which most of the computation expense comes from tokens encoded by the ViT model, the time and space complexity for the cross attention layers are $\mathcal{O}\left(r^{l}N_{\mathrm{vision}}N_{\mathrm{text}}h + N_{\mathrm{text}}h^2\right)$
and $\mathcal{O}\left(r^{l}N_{\mathrm{vision}}N_{\mathrm{text}} + N_{\mathrm{text}}h\right)$ \respn

\section{Performance of ToMe with Different Token Merging Schedules}\label{sec:merging-schedule}

In the original ToMe paper, the authors proposed a merging schedule that involves reducing tokens in each layer by a fixed $k$ tokens per layer. However, as illustrated in Figure \ref{fig:ablation}, we showed that this merging schedule is suboptimal for off-the-shell performance, and this section provides empirical results to confirm this claim.

In this experiment, we exclusively compare two versions of the BSM algorithms utilized in the ToMe paper: one that preserves a percentage $r$ of tokens in each layer and another that reduces a fixed $k$ tokens in each layer. However, for a more comprehensive comparison, we extensively apply these algorithms across 6 ViT backbones (DeiT-T, DeiT-S, DeiT-B, MAE-B, MAE-L, MAE-H) for image classification tasks using the Imagenet-1k dataset, as well as across 4 backbones for image-text retrieval tasks (CLIP-B, CLIP-L, BLIP, BLIP2) on FLickr30k and MSCOCO dataset.
\begin{figure}[t]
\centering
 \begin{minipage}[ht]{.5\linewidth}
    \centering
    \includegraphics[width=\textwidth]{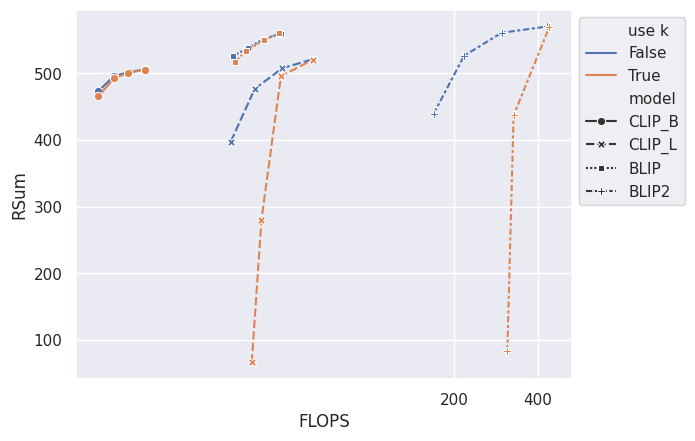}
    \subcaption{MSCOCO}
    \label{fig:coco_use_k_image_text_retrieval}
  \end{minipage}%
  \begin{minipage}[ht]{.5\linewidth}
    \centering
    \includegraphics[width=\textwidth]{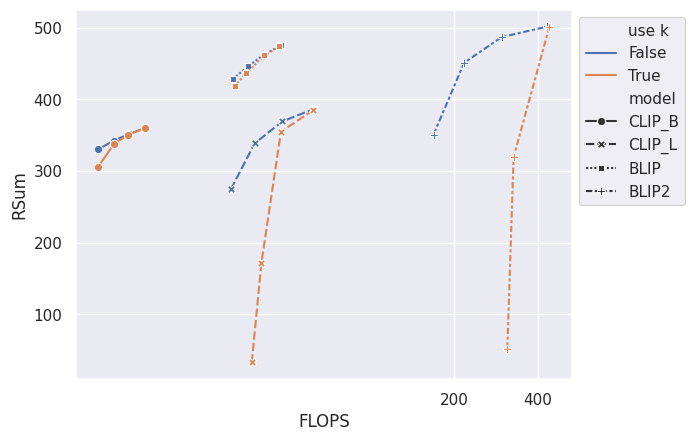}
    \subcaption{Flickr30k}
     \label{fig:flickr_use_k_image_text_retrieval}
  \end{minipage}
  \caption{\hlc[cyan!30]{\textbf{Off-the-shelf}} performance of all backbones for \textbf{image-text retrieval} task using different token merging schedules.}
\end{figure}

\begin{figure}[t]
\centering
 \begin{minipage}[t]{.48\linewidth}
    \centering
    \includegraphics[width=\textwidth]{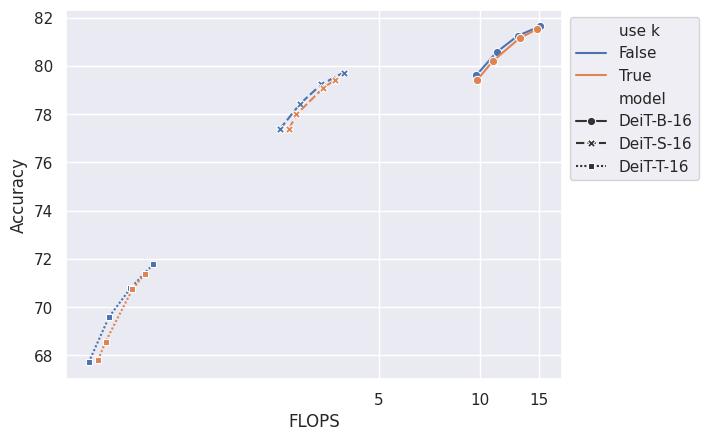}
    \subcaption{DEIT backbones models}
    \label{fig:coco_use_k}
  \end{minipage}%
  \begin{minipage}[t]{.48\linewidth}
    \centering
    \includegraphics[width=\textwidth]{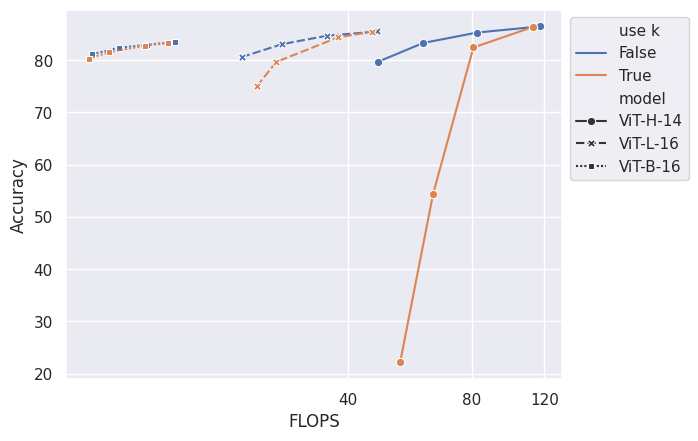}
    \subcaption{MAE backbone models}
     \label{fig:flickr_use_k}
  \end{minipage}
  \caption{\hlc[cyan!30]{\textbf{Off-the-shell}} performance of all backbones for \textbf{image classification} task using different token merging schedules. }
\end{figure}
From Figures \ref{fig:flickr_use_k}, \ref{fig:coco_use_k}, \ref{fig:flickr_use_k_image_text_retrieval}, and \ref{fig:coco_use_k_image_text_retrieval}, it is evident that, given the same FLOPS, the BSM version that uses the remaining percent $r$ shows a clear advantage, outperforming the original schedule by a large margin. This gap becomes even more pronounced when benchmarked on large models like MAE-L, MAE-H, and BLIP-2. During the experiment, we also observed that models compressed using the ratio $r$ tend to run a little faster since a large number of tokens are removed in earlier layers.
\vspace{-0.1in}
\section{Additional Experiments on Text Classification Task.}\label{sec:text-appendix}
This section comprehensively assesses the performance of \ptome against other baseline algorithms for sentiment text classification tasks. We conducted this experiment using the IMDb and SST-2, applying the compression algorithms to the first three layers of the model. To ensure consistency in our experimental results, we used two different backbones: BERT with 12 layers and DistilBERT with 6 layers. Since we only compress the first three layers, in addition to the baselines previously used throughout this paper, we introduced two additional baselines: ToFu-p and ToFu-m. ToFu-p uses BSM to prune tokens instead of merging them, while ToFu-m represents the original ToFu algorithm without pruning.
\begin{table}[!ht]
\centering
\caption{\textbf{Performance} of \ptome versus baselines algorithms when training \textbf{BERT and DistiledBERT} \textbf{when \scolorbox{gray!15}{retrained}  from scratch}.}
\hspace{0.05in}
 \resizebox{\textwidth}{!}{
\begin{tabular}{l|l|l
|cccc|cccc|cccc}
\toprule
\multirow{3}{*}{\textbf{Dataset}} & \multirow{3}{*}{\textbf{Model}} & \multirow{3}{*}{\makecell{\textbf{Compress}\\\textbf{method}}} & \multicolumn{4}{c|}{$r=0.8$} & \multicolumn{4}{c|}{$r=0.75$} & \multicolumn{4}{c}{$r=0.7$}\\ 
\cline{4-15}
& & & \textbf{accuracy} & \textbf{flops} &  \makecell{\textbf{eval}\\ \textbf{speed}} & \makecell{\textbf{train}\\ \textbf{speed}} & \textbf{accuracy} & \textbf{flops} & \makecell{\textbf{eval}\\ \textbf{speed}} &\makecell{\textbf{train} \\\textbf{speed}}& \textbf{accuracy} & \textbf{flops} & \makecell{\textbf{eval}\\\textbf{speed}} & \makecell{\textbf{train} \\\textbf{speed}} \\
\midrule
\multirow{10}{*}{SST-2} &\multirow{5}{*}{BERT} 
    & ToMe & 91.25 (-0.20) & x1.88 & x1.14 &x1.40 & 89.12 (-2.33) & x2.27& x1.22 & x1.56 & 88.00 (-3.45)& x2.72 & x1.30 & x1.68\\
 & & ToFu & 89.82 (-1.83)& x1.88 & x1.15 & x1.40&88.64 (-2.81)& x2.27 & x1.22 & x1.56 & 87.22 (-4.23)& x2.72 & x1.29 & x1.68\\
 & & DCT& 90.66 (-0.79)& x1.81 & x1.05 & x1.12 &89.31 (-2.14)& x2.13 & x1.12 & x1.43 &87.76 (-3.69)& x2.72 & x1.19 & x1.61\\
 & & DiffRate&89.72 (-1.73)&x1.88& x1.14 & x1.40 &87.96 (-3.49)&x2.27 & x1.21 & x1.55 & 87.64 (-3.81)& x2.72 & x1.22 & x1.68\\
 & & \cellcolor{gray!15} \textbf{\ptome}&\cellcolor{gray!15}\textbf{91.72 (+0.27)}&\cellcolor{gray!15}\textbf{x1.88}& \cellcolor{gray!15}\textbf{x1.19} & \cellcolor{gray!15}x1.39 &\cellcolor{gray!15}\textbf{90.28 (-1.17)}&\cellcolor{gray!15}\textbf{x2.27}& \cellcolor{gray!15}\textbf{x1.23} & \cellcolor{gray!15}x1.55 &\cellcolor{gray!15}\textbf{88.67 (-2.78)}&\cellcolor{gray!15}\textbf{x2.72}& \cellcolor{gray!15}x1.25 & \cellcolor{gray!15}x1.65 \\
\cline{2-15}
& \multirow{5}{*}{DistiledBERT} 
    & ToMe & 89.64 (-1.74) & x1.61 & x1.01 & x1.28 & 88.56 (-2.82)& x1.88 & x1.03 & x1.35 & 88.64 (-2.74)& x2.26 & x1.07 & x1.46 \\
 & & ToFu & 89.92 (-1.46) & x1.61 & x1.02 & x1.27 &88.85 (-2.53)& x1.88 & x1.04 & 
 x1.35 &88.76 (-2.62) & x2.26 & x1.07 & x1.46\\
 & & DCT& 89.59 (-1.79) & x1.51& x0.95 & x1.16 &88.08 (-3.30)& x1.61 & x1.00 & 
 x1.22 &87.70 (-3.68)& x1.88 & x1.02& x1.32\\
 & & DiffRate& 89.65 (-1.73) & x1.61 & x1.01 & x1.26& 89.05 (-2.33)& x1.88 & x1.02 & x1.33 & 87.88 (-3.50)& x2.26 & x1.07& x1.42\\
  & & \cellcolor{gray!15} \textbf{\ptome}& \cellcolor{gray!15}\textbf{90.31 (-1.07)} & \cellcolor{gray!15}\textbf{x1.61} & \cellcolor{gray!15}x1.02& \cellcolor{gray!15}x1.25 & \cellcolor{gray!15}\textbf{89.28 (-2.10)} & \cellcolor{gray!15}\textbf{x1.88} & \cellcolor{gray!15}\textbf{x1.05} & \cellcolor{gray!15} x1.34 & \cellcolor{gray!15}\textbf{88.89 (-2.49)} & \cellcolor{gray!15}\textbf{x2.26}  & \cellcolor{gray!15}\textbf{x1.10} & \cellcolor{gray!15}x1.41\\
 \midrule
\multirow{10}{*}{IMDb}& \multirow{5}{*}{BERT} 
    & ToMe & 93.25 (-0.75) & x1.91& x1.68 & x1.86 & 93.01 (-0.99) & x2.30 & x1.88 & x2.06 & 92.33 (-1.67)& x2.77& x2.10 & x2.36\\
 & & ToFu & 93.36 (-0.64)& x1.92 & x1.68 &x1.86&92.99 (-1.01)& x2.30 & x1.89& x2.06 & 92.34 (-1.66)&x2.77 &x2.09& x2.34\\
 & & DCT& 92.39 (-1.61)& x1.90 & x1.60& x1.79& 92.22 (-1.78)& x2.30 & x1.75 &x1.94 & 91.31 (-2.69)&x2.77& x1.97& x2.25\\
 & & DiffRate& 92.96 (-1.04)& x1.91 & x1.68& x1.86 & 92.53 (-1.47)& x2.30 & x1.86 & x2.03 & 92.10 (-1.90)& x2.77 & x2.08& x2.33 \\
 & & \cellcolor{gray!15} \textbf{\ptome}& \cellcolor{gray!15}\textbf{93.52 (-0.48)} & \cellcolor{gray!15}\textbf{x1.91} & \cellcolor{gray!15} x1.66& \cellcolor{gray!15}x1.84 &\cellcolor{gray!15} \textbf{93.27 (-0.73)} & \cellcolor{gray!15}\textbf{x2.30} & \cellcolor{gray!15}x1.84 & \cellcolor{gray!15}x2.01 & \cellcolor{gray!15}\textbf{92.74 (-1.26)} & \cellcolor{gray!15}\textbf{x2.77}  & \cellcolor{gray!15}x2.08 & \cellcolor{gray!15}x2.31\\
\cline{2-15}
& \multirow{5}{*}{DistiledBERT}
    & ToMe & 92.45 (-0.55) & x1.83 & x1.47 & x1.57 & 92.34 (-0.66)& x2.15 & x1.57 & x1.64 & 91.86 (-1.14)& x2.53 & x1.69 & x1.86\\
 &  & ToFu & 92.55 (-0.45)& x1.83 & x1.47 & x1.57 & 92.33 (-0.67)& x2.15& x1.57 &  
x1.64& 91.91 (-1.09)& x2.53 & x1.69& x1.85\\
 & & DCT & 92.38 (-0.62)& x1.61 &x1.39 & x1.50 & 91.87 (-1.13)& x1.79& x1.48 & x1.62 & 91.69 (-1.31)& x2.19 & x1.56 & x1.77 \\
 &  & DiffRate&92.43 (-0.57) & x1.83 & x1.44 & x1.55 & 92.16 (-0.84)& x2.15 & x1.57 & x1.62& 91.78 (-1.22)& x2.53&x1.63&x1.84\\
 & & \cellcolor{gray!15}\textbf{\ptome}& \cellcolor{gray!15}\textbf{92.71 (-0.29)} & \cellcolor{gray!15}\textbf{x1.83} & \cellcolor{gray!15}x1.43& \cellcolor{gray!15}x1.54 & \cellcolor{gray!15}\textbf{92.55 (-0.45)} & \cellcolor{gray!15}\textbf{x2.15} & \cellcolor{gray!15}x1.54 & \cellcolor{gray!15}x1.61& \cellcolor{gray!15}\textbf{92.06 (-0.94)} & \cellcolor{gray!15}\textbf{x2.53} & \cellcolor{gray!15}x1.62& \cellcolor{gray!15}x1.83 \\
\bottomrule
\end{tabular}}\label{tab:train_text}
\end{table}
\begin{figure}[b]
\centering
\includegraphics[width=\textwidth]{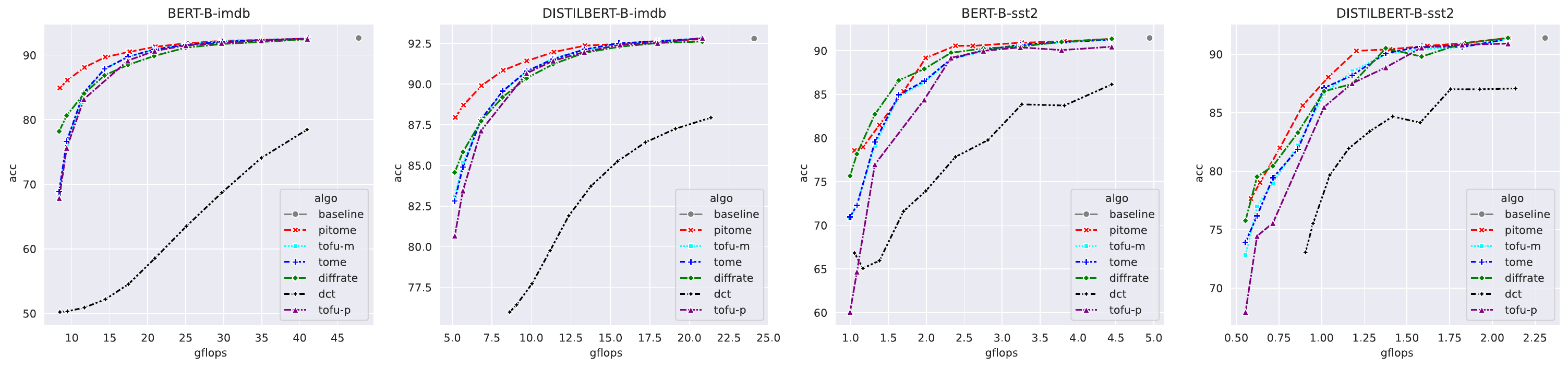}  
\captionof{figure}{\hlc[cyan!30]{\textbf{Off-the-shelf}} performance of various algorithms on the \textbf{text classification task}.
}
\label{fig:text_ots}
\end{figure}

As demonstrated in Table \ref{tab:train_text} and Figure \ref{fig:text_ots}, our findings align with empirical results from previous tasks, indicating that \ptome consistently achieves superior performance compared to other baselines. Particularly noteworthy is the performance on the IMDb dataset with a large context length, where even after reducing FLOPs by 80\%, models compressed by \ptome still maintain off-the-shelf accuracy above 85\%, while other baseline algorithms see the off-the-shelf accuracy drop below 70\%. Furthermore, Table \ref{tab:train_text} highlights that our algorithm can also facilitate better model learning compared to other methods, achieving high accuracy that closely approaches that of the original model.

\section{Proof of Theorem \ref{theorem_spectrum_consistent_PiToME}}\label{section_proof_theorem_spectrum_consistent_PiToME}
\subsection{Sketch of Proof}

The proof sketch for Theorem \ref{theorem_spectrum_consistent_PiToME} begins by defining coarsened and lifted versions of the original graph $\cG$ (see Definitions \ref{def_Gc} and \ref{def_Gl}) using the \ptome and ToMe algorithms. The goal is to demonstrate that the spectral distance between $\cG$ and its \ptome-coarsened counterpart converges to zero, whereas the distance for ToMe remains bounded away from zero. The sketch proceeds by introducing Propositions \ref{proposition_spectrum_consistent_PiToME_ToMe_convergence}, \ref{proposition_spectrum_consistent_PiToME}, and \ref{proposition_spectrum_consistent}, which build toward the main result:
\begin{enumerate}
    \item Proposition \ref{proposition_spectrum_consistent_PiToME_ToMe_convergence} establishes upper bounds on the edge weight differences between merged nodes under the \ptome and ToMe methods.

    \item Proposition \ref{proposition_spectrum_consistent_PiToME} employs standard mild assumptions to relate the cosine similarity among nodes within clusters, demonstrating that the upper bound error $\epsilon\ps_\text{\ptome}$ between the edge weights of merged nodes in \ptome-coarsened graphs $\cG_{\text{\ptome}}^{(n)}$ converges to $0$, whereas this property does not hold for the ToMe-coarsened graphs $\cG_{\text{ToMe}}^{(n)}$.

    \item Proposition \ref{proposition_spectrum_consistent} completes the sketch by bounding the spectral distances of the coarsened graphs in terms of the edge weight differences from the earlier Propositions \ref{proposition_spectrum_consistent_PiToME_ToMe_convergence} and \ref{proposition_spectrum_consistent_PiToME}.
\end{enumerate}
Each proposition is proven in detail within the appendix sections following this sketch. Note that Theorem \ref{theorem_spectrum_consistent_PiToME} immediately follows the following Propositions \ref{proposition_spectrum_consistent_PiToME_ToMe_convergence}, \ref{proposition_spectrum_consistent_PiToME}, and \ref{proposition_spectrum_consistent}, which are proved \resp in Sections \ref{sec_lemma_spectrum_consistent_PiToME_ToMe_convergence_proof}, \ref{sec_lemma_spectrum_consistent_PiToME_proof}, and \ref{sec_lemma_spectrum_consistent_proof}.


\begin{proposition}\label{proposition_spectrum_consistent_PiToME_ToMe_convergence}
     Suppose the graphs $\cG\ps_0$, $\cG_{\text{\ptome}}\ps$, and $\cG_{\text{ToMe}}\ps$ are coarsened from the original graph $\cG$ by iteratively merging pairs of nodes $v_{a_s}$ and $v_{b_s}$ \wrt the true partition $\cP\ps_{0}=\{\cV\ps_{0i}\}_{i \in [s]}$, the \ptome-partition $\cP\ps_{\text{\ptome}}=\{\cV\ps_{\text{\ptome}i}\}_{i\in[s]}$, defined by \ptome Algorithm \ref{algo:pitome}, and the ToMe-partition~\citep{tome,kim2024token}, $\cP\ps_{\text{ToMe}}=\{\cV\ps_{\text{ToMe}i}\}_{i\in[s]}$, for $s = N, \ldots,n+1$.
      We assume the following standard mild assumption:
     \begin{itemize}

         \item[(A2).] There exists a margin $m$ \st
         \begin{align}\label{eq_A2_appendix}
             \cos(v_{a_s}, v_{b_s}) \ge m > \cos(v_{a_s}, v_{c_s}), \quad \forall v_{a_s} \in \cV^{(s)}_{0i}, \forall v_{b_s} \in \cV^{(s)}_{0i}, \forall v_{c_s} \in \cV^{(s)}_{0j}, \forall i \neq j \in [s].
         \end{align}
     \end{itemize}
    Then, the edge weights of merged nodes from \ptome Algorithm \ref{algo:pitome} and the ToMe-partition~\citep{tome,kim2024token} satisfy
    \begin{align}\label{eq_lemma_spectrum_consistent_PiToME_ToMe_convergence}
        \|\bfW[a_s,:]-\bfW[b_s,:]\|_1 \le \epsilon\ps
    \end{align}
     for some nonnegative upper bounds $\epsilon\ps$, $s = N,N-1,\ldots,n+1$ defined as follows:
     \begin{align}\label{eq_epsilon_s_define}
     \epsilon\ps
     = \begin{cases}
          2(1-\cos(v_{a_s},v_{b_s}))& \text{ if } v_{a_s}\in \cV\ps_{0i}, v_{b_s} \in \cV\ps_{0i}, \forall i \in [s],\\
         3(1-\beta)& \text{ if } v_{a_s} \in \cV\ps_{0i}, v_{b_s} \in \cV\ps_{0j}, i\neq j \in [s].
     \end{cases}.
\end{align}
\end{proposition}

\begin{proposition}\label{proposition_spectrum_consistent_PiToME}
     We assume some standard mild assumptions:
     \begin{itemize}
         \item[(A1).] $\mathbb{E}[\cos(v_{a_s}, v_{b_s})] \rightarrow 1, \quad \forall v_{a_s} \in \cV^{(s)}_{0i}, \forall v_{b_s} \in \cV^{(s)}_{0i}, i \in [s]$.

         \item[(A2).] There exists a margin $m$ s.t.,  $\cos(v_{a_s}, v_{b_s}) \ge m > \cos(v_{a_s}, v_{c_s}), \quad \forall v_{a_s} \in \cV^{(s)}_{0i}, \forall v_{b_s} \in \cV^{(s)}_{0i}, \forall v_{c_s} \in \cV^{(s)}_{0j}, \forall i \neq j \in [s]$.

         \item[(A3).]  There is an order of cardinality in the true partition, without loss of generality, we assume $N\ps_1 \ge N\ps_2 \ge \ldots \ge N\ps_s$, where $N\ps_i = |\cV^{(s)}_{0i}|, \forall i \in [s]$.
     \end{itemize}
     Here $\Ep{\cdot}$ stands for the expectation of the random variables that define the random events $M\ps_{ab} := \{v_{a_s}\in \cV\ps_{0i}, v_{b_s} \in \cV\ps_{0i}, i \in [s]\}$, indicating that the two merged nodes belong to the same true partition $\cV\ps_{0i}$.
     Given the $\epsilon\ps$ defined in equation (\ref{eq_epsilon_s_define}), we obtain
     \begin{enumerate}
        \item The upper bound error $\epsilon\ps_\text{\ptome}$ between the edge weights of merged nodes from \ptome-coarse  $\cG_{\text{\ptome}}^{(n)}$ graphs converges to $0$, \ie
        \begin{align*}
            \text{$\epsilon\ps_\text{\ptome} \rightarrow 0$ 
        as $\mathbb{E}(\cos(v_{a_s},v_{b_s})) \rightarrow 1$, for any $s = N, \ldots,n+1$.}
        \end{align*}
         
        \item The upper bound error $\epsilon\ps_{\text{ToMe}}$ between the edge weights of merged nodes from ToMe-coarse  $\cG_{\text{ToMe}}^{(n)}$ graphs converges to a non-negative constant $C\ps$, with a high probability that $C\ps > 0$, \ie
        \begin{align*}
            \epsilon\ps_{\text{ToMe}} \rightarrow C\ps \text{ and } \Pb(C\ps>0) >0 
       \text{ as $\mathbb{E}(\cos(v_{a_s},v_{b_s})) \rightarrow 1$, for any $s = N, \ldots,n+1$.}
        \end{align*}
    \end{enumerate}
\end{proposition}
\begin{proposition}\label{proposition_spectrum_consistent}
     Suppose the graphs $\cG\ps_0$, $\cG_{\text{\ptome}}\ps$, and $\cG_{\text{ToMe}}\ps$ are coarsened from the original graph $\cG$ by iteratively merging pairs of nodes $v_{a_s}$ and $v_{b_s}$ \wrt the true partition $\cP\ps_{0}=\{\cV\ps_{0i}\}_{i \in [s]}$, the \ptome-partition $\cP\ps_{\text{\ptome}}=\{\cV\ps_{\text{\ptome}i}\}_{i\in[s]}$, defined by \ptome Algorithm \ref{algo:pitome}, and the ToMe-partition~\citep{tome,kim2024token}, $\cP\ps_{\text{ToMe}}=\{\cV\ps_{\text{ToMe}i}\}_{i\in[s]}$, for $s = N, \ldots,n+1$.
    If the edge weights of merged nodes satisfy $$\|\bfW[a_s,:]-\bfW[b_s,:]\|_1 \le \epsilon\ps$$ for some nonnegative upper bounds $\epsilon\ps$, $s = N,N-1,\ldots,n+1$. 
    The spectral distances between the original $\cG\equiv\cG^{(N)}_{0}$ and the \ptome-coarse and ToMe-coarse $\cG_{\text{\ptome}}^{(n)}$ graphs are bounded as follows:
    \begin{align*}
        \max\{\sd(\cG,\cG_{\text{\ptome}}^{(n)}),\sd(\cG,\cG_{\text{ToMe}}^{(n)})\}& \le \frac{3N}{2}\sum_{s=N}^{n+1}\epsilon\ps.
    \end{align*}
\end{proposition}
\subsection{Proof of Proposition \ref{proposition_spectrum_consistent_PiToME_ToMe_convergence}}\label{sec_lemma_spectrum_consistent_PiToME_ToMe_convergence_proof}
We want to prove that the edge weights of merged nodes from \ptome Algorithm \ref{algo:pitome} and the ToMe-partition~\citep{tome,kim2024token} satisfy 
\begin{align*}
    \|\bfW[a_s,:]-\bfW[b_s,:]\|_1 \le \epsilon\ps
\end{align*}
for some nonnegative upper bounds $\epsilon\ps$, $s = N,N-1,\ldots,n+1$.

Let us start with the simplest case, where we suppose the graph $\cG_c$ is coarsened from $\cG$ by merging one pair of nodes $v_{a}$ and $v_{b}$ using \ptome Algorithm \ref{algo:pitome} and the ToMe-partition.
We can then demonstrate that the edge weights of the merged nodes satisfy certain conditions (for more details, refer to Appendix \ref{proof_eq_PiToMe_norm1_2Nodes}):
\begin{align}\label{eq_PiToMe_norm1_2Nodes}
    \|\bfW[a,:]-\bfW[b,:]\|_1 \le \epsilon \text{ for a nonnegative upper bound $\epsilon$}.
\end{align}

Recall that the graphs $\cG\ps_0$, $\cG_{\text{\ptome}}\ps$, and $\cG_{\text{ToMe}}\ps$ are coarsened from the original graph $\cG$ by iteratively merging pairs of nodes $v_{a_s}$ and $v_{b_s}$ \wrt the true partition $\cP\ps_{0}=\{\cV\ps_{0i}\}_{i \in [s]}$, the \ptome-partition $\cP\ps_{\text{\ptome}}=\{\cV\ps_{\text{\ptome}i}\}_{i\in[s]}$, defined by \ptome Algorithm \ref{algo:pitome}, and the ToMe-partition~\citep{tome,kim2024token}, $\cP\ps_{\text{ToMe}}=\{\cV\ps_{\text{ToMe}i}\}_{i\in[s]}$, for $s = N, \ldots,n+1$.
By iteratively applying the merged 2-nodes inequalities  (\ref{eq_PiToMe_norm1_2Nodes}), we obtain the desired upper bound as follows:
\allowdisplaybreaks
\begin{align*}
    \|\bfW[a_s,:]-\bfW[b_s,:]\|_1 \le \epsilon\ps, \text{ for all $s = N,N-1,\ldots,n+1$,}
\end{align*}
\allowdisplaybreaks
where 
\begin{align*}
     \epsilon\ps
     = \begin{cases}
          2(1-\cos(v_{a_s},v_{b_s}))& \text{ if } v_{a_s}\in \cV\ps_{0i}, v_{b_s} \in \cV\ps_{0i}, i \in [s],\\
         3(1-\beta)& \text{ if } v_{a_s} \in \cV\ps_{0i}, v_{b_s} \in \cV\ps_{0j}, i\neq j \in [s].
     \end{cases}.
\end{align*}
\subsection{Proof of Proposition \ref{proposition_spectrum_consistent_PiToME}}\label{sec_lemma_spectrum_consistent_PiToME_proof}
Recall that in equation (\ref{eq_epsilon_s_define}) of Proposition \ref{proposition_spectrum_consistent_PiToME_ToMe_convergence}, we defined $\epsilon\ps$, for $s = N,N-1,\ldots,n+1$, as follows:
 \begin{align*}
 \epsilon\ps
     = \begin{cases}
          2(1-\cos(v_{a_s},v_{b_s}))& \text{ if } v_{a_s}\in \cV\ps_{0i}, v_{b_s} \in \cV\ps_{0i}, i \in [s],\\
         3(1-\beta)& \text{ if } v_{a_s} \in \cV\ps_{0i}, v_{b_s} \in \cV\ps_{0j}, i\neq j \in [s].
     \end{cases}.
\end{align*}
Recall that we defined $\Ep{\cdot}$ as the expectation of the random variables that define the random events $M\ps_{ab} := \{v_{a_s}\in \cV\ps_{0i}, v_{b_s} \in \cV\ps_{0i}, i \in [s]\}$, indicating that the two merged nodes belong to the same true partition $\cV\ps_{0i}$.
Using the definition and the linearity of expectation, we have
\begin{align}\label{eq_Expectation_epsilon_2}
\Ep{\epsilon\ps} = 2[1-\Ep{\cos(v_{a_s},v_{b_s})}]+3(1-\beta)(1-\Pb(M\ps_{ab})).
\end{align}
Under Assumptions (A2) and (A3) and the energy-based merging mechanism of the \ptome algorithm, we can verify that $\Pb(M\ps_{ab}) = 1$ via Lemma \ref{lemma_PiToMe_contradiction}, which is proved in Appendix \ref{proof_lemma_PiToMe_contradiction}. Refer to Lemma \ref{lemma_ToMe_contradiction} for the ToMe approach (proof in Appendix \ref{proof_lemma_PiToMe_contradiction}), noting that there exists a scenario where the random event $M\ps_{ab}$ does not occur, leading to $\Pb(M\ps_{ab}) < 1$.
This leads to the desired results.
\begin{lemma}\label{lemma_PiToMe_contradiction}
    For the \ptome approach, it holds that $\Pb(M\ps_{ab}) = 1$, \ie~there exists a true partition $\cV\ps_{0i}$, $i\in[s]$, such that $v_{a_s} \in \cV\ps_{0i}, v_{b_s} \in \cV\ps_{0i}$. 
\end{lemma}
\begin{lemma}\label{lemma_ToMe_contradiction}
    For the ToMe approach, there is a case where the random event $M\ps_{ab}$ does not occur and therefore $\Pb(M\ps_{ab}) < 1$. 
\end{lemma}
\subsection{Proof of Proposition \ref{proposition_spectrum_consistent}} \label{sec_lemma_spectrum_consistent_proof}
    Let us start with the simplest case, where we suppose the graph $\cG_c$ is coarsened from $\cG$ by merging one pair of nodes $v_{a}$ and $v_{b}$.
    We can then prove the following property: If the edge weights of merged nodes satisfy $$\|\bfW[a,:]-\bfW[b,:]\|_1 \le \epsilon$$ then the spectral distance between the original and lifted graphs is bounded by
   \begin{align}
       \|\bslambda-\bslambda_l\|_1 \le \frac{3}{2}\epsilon \label{eq_Gc_norm1_2Nodes}.
   \end{align}
    \vspace{-0.2in}
    \begin{proof}[Proof of the 2-node triangle inequality  \ref{eq_Gc_norm1_2Nodes}]
    Indeed, since the coarse graph $\cG_c$ is coarsened by merging a pair of nodes $v_a$ and $v_b$, the edge weights of the lifted graph $\cG_l$ can be formulated as:
    \setlength{\abovedisplayskip}{3pt} 
\setlength{\belowdisplayskip}{3pt}
    \begin{align}\label{eq_Wl_def}
        \bfW_l[i,j] = \begin{cases}
            \frac{\bfW[a,a]+2\bfW[a,b]+\bfW[b,b]}{4} & \text{ if } i,j \in \{a,b\},\\
            \frac{\bfW[a,j]+\bfW[b,j]}{2} & \text{ if } i \in \{a,b\} \text{ and } j \notin \{a,b\},\\
            \frac{\bfW[i,a]+\bfW[i,b]}{2} & \text{ if } i \notin \{a,b\} \text{ and } j \in \{a,b\},\\
            \bfW[i,j]& \text{ otherwise.}
        \end{cases}
    \end{align}
    Here we use the fact that the adjacency matrix elements of $\cG_l$ are given by
    \begin{align}\label{eq_Wl_adjacency_matrix}
        \bfW_l[i,j]  = \frac{\sum_{v_i \in \cV_i}\sum_{v_j \in \cV_j}\bfW[i,j]}{|\cV_i||\cV_j|}\text{ where $v_i \in \cV_i$, and $v_j \in \cV_j$.}
    \end{align} 
The corresponding node degree of $\cG_l$ is
\begin{align}\label{eq_dl_def}
    d_{li} = \begin{cases}
        \frac{d_a + d_b}{2} &\text{ if } i\in\{a,b\},\\
        d_i &\text{ otherwise.}
    \end{cases}
\end{align}
Using the result from Lemma \ref{lem_Weyl_inequality}, we can bound on the eigenvalue gap between $\bslambda$ and $\bslambda_l$ via the perturbation matrix $\bfE = \bfL - \bfL_l= \bfD - \bfD_l + \bfW_l - \bfW$ as follows:
\begin{align}\label{eq_lambda_E}
     \|\bslambda-\bslambda_l\|_1 = \sum_{i=1}^N|\lambda_i-\lambda_{li}| \le N\|\bfE\|_2  \le N\sqrt{\|\bfE\|_\infty \|\bfE\|_1}.
\end{align}
Here the last inequality follows from the upper bound of the spectral norm $\|\bfE\|_2$ of \cite{wolkowicz_bounds_1980}, and we defined $\|\bfE\|_\infty = \max_{i}\sum_{j}|\bfE[i,j]|$, which is simply the maximum absolute row sum of the matrix, and $\|\bfE\|_1 = \max_{j} \sum_{i}|\bfE[i,j]|$, which is simply the maximum absolute column sum of the matrix.
\begin{lemma}[Perturbations of eigenvalues: Weyl's inequality from \cite{weyl_asymptotische_1912}]\label{lem_Weyl_inequality}
    Let $\bfL \in \R^{N \times N}$ and $\bfL_l \in \R^{N \times N}$ be symmetric matrices.  Then for all $i=1,\ldots,N$,
    \begin{align}
        \max_{i = 1,\ldots,N}|\bslambda_{li}(\bfL_l)-\bslambda_{i}(\bfL)| \le \|\bfL_l - \bfL\|_2,
    \end{align}
    where $\|\bfL_l - \bfL\|_2$ is the induced 2-norm of $\bfL_l - \bfL$.
\end{lemma}
The equations (\ref{eq_Wl_def}), (\ref{eq_Wl_adjacency_matrix}), and (\ref{eq_dl_def}) yield the following identities:
\allowdisplaybreaks
\begin{align}\label{eq_WlW_def}
    \bfW_l[i,j] - \bfW[i,j] &= \begin{cases}
        \frac{\bfW[a,a]+2\bfW[a,b]+\bfW[b,b]}{4} -\bfW[i,j] & \text{ if } i,j \in \{a,b\},\\
        \frac{\bfW[a,j]+\bfW[b,j]}{2}-\bfW[i,j] & \text{ if } i \in \{a,b\} \text{ and } j \notin \{a,b\},\\
        \frac{\bfW[i,a]+\bfW[i,b]}{2}-\bfW[i,j] & \text{ if } i \notin \{a,b\} \text{ and } j \in \{a,b\},\\
        0& \text{ otherwise.}
    \end{cases}\\
    \vspace{-0.05in}
    \text{ and }\hspace{3cm}&\nonumber\\
    \bfD[i,i] - \bfD_{l}[i,i] &= \begin{cases}
        d_i - \frac{d_a + d_b}{2} &\text{ if } i\in\{a,b\},\\
        0 &\text{ otherwise.}
    \end{cases}    
\end{align}
Now we want to prove that $\|\bfE\|_\infty \le \epsilon$ and $\|\bfE\|_1 \le \epsilon$. Let us first focus on the first term $\|\bfE\|_\infty =\max_{i}\sum_{j}|\bfE[i,j]|$.
Via the triangle inequality, we have
\begin{align}\label{eq_norm_infity_E}
    \|\bfE\|_\infty \le  \|\bfD-\bfD_l\|_\infty +  \|\bfW_l-\bfW\|_\infty.
\end{align}
Then, again using the triangle inequality and the assumption of Proposition \ref{proposition_spectrum_consistent}, we obtain the first upper bound on the first term of the equation (\ref{eq_norm_infity_E}) as follows:
\allowdisplaybreaks
\begin{align}\label{eq_DDl_infinity}
     \|\bfD-\bfD_l\|_\infty &= \max_{i}\sum_{j}|\bfD[i,j]-\bfD_l[i,j]| = \max_{i \in\{a,b\}} \left|d_i-\frac{d_a+d_b}{2}\right|\nonumber\\
     &=\left|d_a-\frac{d_a+d_b}{2}\right| = \left|\frac{d_a-d_b}{2}\right|\nonumber\\
     &= \frac{1}{2} \left|\sum_{j=1}^N \bfW_{aj} - \sum_{j=1}^N \bfW_{bj}\right| \le \frac{1}{2} \sum_{j=1}^N \left|\bfW[a,j] -  \bfW[b,j]\right|\nonumber\\
     &=\frac{1}{2} \left\|\bfW[a,:] -  \bfW[b,:]\right\|_1 \le \frac{\epsilon}{2}.
\end{align}
For the second upper bound term from the equation (\ref{eq_norm_infity_E}), we consider two cases for each index $i = 1,\ldots,N$: $i\in\{a,b\}$ ({\bf \textit{Case 1}}) and $i\notin\{a,b\}$ ({\bf \textit{Case 2}}).

\textbf{\textit{Case 1.}} Assume that $i\in\{a,b\}$. Since $a$ and $b$ have the same role, we can take $i=a$ without loss of generality. Using the equation (\ref{eq_WlW_def}), it holds that
\allowdisplaybreaks

\begin{align}
   & \|\bfW_l-\bfW\|_\infty = \max_{i\in\{a,b\}}\sum_{j}|\bfW_l[i,j]-\bfW[i,j]|\nonumber\\
    &= \left|\frac{2\bfW[a,b]+\bfW[b,b]-3\bfW[a,a]}{4}\right| +\left|\frac{\bfW[a,a]+\bfW[b,b]-2\bfW[a,b]}{4}\right|
   + \sum_{j\notin\{a,b\}} \left|\frac{\bfW[a,j]-\bfW[b,j]}{2} \right|\nonumber\\
    &= \left|\frac{(\bfW[b,b]-\bfW[a,b])+3(\bfW[a,b]-\bfW[a,a])}{4}\right|\nonumber\\
    & \quad+\left|\frac{(\bfW[a,a]-\bfW[a,b])+(\bfW[a,b]-\bfW[b,b])+2(\bfW[b,b]-\bfW[a,b])}{4}\right|\nonumber\\
    &\quad + \sum_{j\notin\{a,b\}} \left|\frac{\bfW[a,j]-\bfW[b,j]}{2} \right|\nonumber.
    \end{align}
Therefore, it holds that
\allowdisplaybreaks
\begin{align}
     \|\bfW_l-\bfW\|_\infty &\le \frac{1}{4}|\bfW[a,b]-\bfW[b,b]|+\frac{3}{4}|\bfW[a,a]-\bfW[a,b]|\nonumber\\
    & \quad+\frac{1}{4}|\bfW[a,a]-\bfW[a,b]|+\frac{1}{4}|\bfW[a,b]-\bfW[b,b]|+\frac{2}{4}|\bfW[a,b]-\bfW[b,b]|\nonumber\\
    &\quad + \frac{1}{2} \sum_{j\notin\{a,b\}} |\bfW[a,j]-\bfW[b,j]| \text{ (using triangle inequalities)}\nonumber\\
    &= |\bfW[a,b]-\bfW[b,b]|+|\bfW[a,a]-\bfW[a,b]| + \frac{1}{2} \sum_{j\notin\{a,b\}} |\bfW[a,j]-\bfW[b,j]|\nonumber\\
    &\le |\bfW[a,b]-\bfW[b,b]|+|\bfW[a,a]-\bfW[a,b]| + \sum_{j\notin\{a,b\}} |\bfW[a,j]-\bfW[b,j]|\nonumber\\
    &= |\|\bfW[a,:]-\bfW[b,:]\|_1 \le \epsilon.\label{eq_WlW_infinity_Inab}
\end{align}

\textbf{\textit{Case 2.}} Assume that $i\notin\{a,b\}$.  Using the equation (\ref{eq_WlW_def}), we obtain
\allowdisplaybreaks
\begin{align}\label{eq_WlW_infinity_Notinab}
    \|\bfW_l-\bfW\|_\infty &= \max_{i\notin\{a,b\}}\sum_{j \in \cV}|\bfW_l[i,j]-\bfW[i,j]|\nonumber\\
    &= \left|\frac{\bfW[i,a]+\bfW[i,b]}{2}-\bfW[i,a]\right| + \left|\frac{\bfW[i,a]+\bfW[i,b]}{2}-\bfW[i,b]\right|\nonumber\\
    &= \left|\frac{\bfW[i,a]-\bfW[i,b]}{2}\right| + \left|\frac{\bfW[i,a]-\bfW[i,b]}{2}\right|\nonumber\\
    &= \left|\bfW[i,a]-\bfW[i,b]\right| \le |\|\bfW[a,:]-\bfW[b,:]\|_1 \le \epsilon.
\end{align}
Combining (\ref{eq_WlW_infinity_Inab}) and (\ref{eq_WlW_infinity_Notinab}), we obtain 
\allowdisplaybreaks
\begin{align*}
     \|\bfW_l-\bfW\|_\infty \le \epsilon.
\end{align*}
This leads to 
\begin{align*}
     \|\bfE\|_\infty \le \frac{3}{2}\epsilon
\end{align*}
when using the inequalities (\ref{eq_norm_infity_E}) and (\ref{eq_DDl_infinity}). Similarly, we can show that $\|\bfE\|_1 \le \frac{3}{2}\epsilon$. Therefore, the equation (\ref{eq_lambda_E}) leads to the desired claim as follows:
\allowdisplaybreaks
\begin{align*}
    \|\bslambda-\bslambda_l\|_1 \le N\sqrt{\frac{3}{2}\epsilon \frac{3}{2}\epsilon} =\frac{3N}{2}\epsilon.
\end{align*}
\end{proof}
\vspace{-0.2in}
Recall that the graphs $\cG\ps_0$, $\cG_{\text{\ptome}}\ps$, and $\cG_{\text{ToMe}}\ps$ are coarsened from the original graph $\cG$ by iteratively merging pairs of nodes $v_{a_s}$ and $v_{b_s}$ \wrt the true partition $\cP\ps_{0}=\{\cV\ps_{0i}\}_{i \in [s]}$, the \ptome-partition $\cP\ps_{\text{\ptome}}=\{\cV\ps_{\text{\ptome}i}\}_{i\in[s]}$, defined by \ptome Algorithm \ref{algo:pitome}, and the ToMe-partition~\citep{tome,kim2024token}, $\cP\ps_{\text{ToMe}}=\{\cV\ps_{\text{ToMe}i}\}_{i\in[s]}$, for $s = N, \ldots,n+1$.
By iteratively applying the triangle inequalities  (\ref{eq_Gc_norm1_2Nodes}), the spectral distances between the original $\cG\equiv\cG^{(N)}_{0}$ and the \ptome-coarse $\cG_{\text{\ptome}}^{(n)}$ and ToMe-coarse $\cG_{\text{ToMe}}^{(n)}$ graphs are bounded as follows:
    \begin{align*}
        \max\{\sd(\cG,\cG_{\text{\ptome}}^{(n)}),\sd(\cG,\cG_{\text{ToMe}}^{(n)})\} &\le \sum_{s=N}^{n+1} \max\{\sd(\cG\ps_0,\cG^{(s-1)}_{\text{\ptome}}),\sd(\cG\ps_0,\cG^{(s-1)}_{\text{ToMe}})\} \nonumber\\
        &\le \frac{3N}{2}\sum_{s=N}^{n+1}\epsilon\ps.\nonumber
    \end{align*}

\subsection{Proofs of Technical Results}

\subsubsection{Proof of the merged 2-nodes inequality  (\ref{eq_PiToMe_norm1_2Nodes})}\label{proof_eq_PiToMe_norm1_2Nodes}
Recall that there exists the graph $\cG_0$ coarsened from the original graph $\cG$ by merging a pair of nodes $v_{a_s}$ and $v_{b_s}$ \wrt the true partition $\cP_{0}=\{\cV_{0i}\}_{i \in [s]}$.
Then, we have $\cV = \cV_{01}\cup\cV_{02}\cup\ldots\cup\cV_{0n}$, where $n=N-1$.
We also note that the energy score $E_a$ of node $v_a$ is calculated using the following equation:
\begin{equation} \label{eq_energy_appendix}
    E_a  = \frac{1}{N} \sum_{b \in \mathcal{N}(a)}  {f}_m(\cos(v_a,v_b)),\hspace{.2cm} {f}_m(x) = \begin{cases}
        x  & \text{ if }  x \geq m \\
        \alpha(\exp(x-m) - 1) & \text{otherwise}
    \end{cases}. 
\end{equation}
Inspired by the construction of the function ${f}_m$ in equation (\ref{eq_energy_appendix}) and according to Assumption (A2) in the inequalities (\ref{eq_A2_appendix}), we can replace the smooth term $\alpha(\exp(x-m)-1)$ by a constant $\beta$ for simplicity.
More precisely, if the nodes $v_a$ and $v_b$ are not considered true neighbours, \ie~their cosine similarity is less than the margin $m$, then we can simplify the expression as follows:
\begin{align}
    \cos(v_a,v_b) = \beta:= \sup_{v_a \in \cV_{0i}, v_b \in \cV_{0j}, i \neq j \in [N]}\alpha(\exp(\cos(v_i,v_j)-m) - 1) <0.
\end{align}
To check the inequality (\ref{eq_PiToMe_norm1_2Nodes}), we examine the following term
\begin{align*}
    \|\bfW[a,:]-\bfW[b,:]\|_1
\end{align*}
in two cases: 

\textbf{\textit{Case 1.}} If two nodes $v_a$ and $v_b$ belong to the same true partition, say for example, $\cV_{0i}, i \in[N]$, then since $n=N-1$, we have $\cV_{0i} = \{v_a,v_b\}$. Therefore, we can expand the previous 1-norm as follows:
\begin{align}
   &  \|\bfW[a,:]-\bfW[b,:]\|_1 = \sum_{k=1}^N | W[a,k] - W[b,k]| \nonumber\\
     &= \sum_{v_k \in \cV_{0i}}^N | W[a,k] - W[b,k]| + \sum_{v_k \notin \cV_{0i}}^N | W[a,k] - W[b,k]| \nonumber\\
      &= | W[a,a] - W[b,a]| + | W[a,b] - W[b,b]| \nonumber\\
     &= 2|1-\cos(v_a,v_b)|.\label{eq_Wab_epsilon1}
\end{align}

\textbf{\textit{Case 2.}} If $v_a \in \cV_{0i}$ and $v_b \in \cV_{0j}$ such that $i \neq j$, $i,j\in[N]$. Since $n=N-1$, we have either $\cV_{0i} = \{v_a,v_{0i}\},\cV_{0j} = \{v_b\}$ ($\mathrm{Case \ 2.1}$) or $\cV_{0i} = \{v_a\},\cV_{0j} = \{v_b,v_{0j}\}$ ($\mathrm{Case \ 2.2}$). Let us first consider the $\mathrm{Case \ 2.1}$, then it holds that
\allowdisplaybreaks
\vspace{0.05in}
\begin{align}\label{eq_Wab_epsilon2_1}
     &\|\bfW[a,:]-\bfW[b,:]\|_1 = \sum_{k=1}^N | W[a,k] - W[b,k]| \nonumber\\
     &= \sum_{v_k \in \cV_{0i}}^N | W[a,k] - W[b,k]| + \sum_{v_k \in \cV_{0j}}^N | W[a,k] - W[b,k]| + \sum_{v_k \notin \cV_{0i},v_k \notin \cV_{0j}}^N | W[a,k] - W[b,k]| \nonumber\\
      &= | W[a,a] - W[b,a]| + | W[a,0i] - W[b,0i]|+ | W[a,b] - W[b,b]|  \nonumber\\
     &= (1-\beta) + |1-\cos(v_a,v_{0i})-(1-\beta)| + (1-\beta) \nonumber\\
     &= 2(1-\beta) +|\cos(v_a,v_{0i})-\beta| \le 3(1-\beta).
\end{align}

Now, let us first consider the $\mathrm{Case \ 2.2}$, then it holds that
\vspace{0.05in}
\begin{align}\label{eq_Wab_epsilon2_2}
     &\|\bfW[a,:]-\bfW[b,:]\|_1 = \sum_{k=1}^N | W[a,k] - W[b,k]| \nonumber\\
     &= \sum_{v_k \in \cV_{0i}}^N | W[a,k] - W[b,k]| + \sum_{v_k \in \cV_{0j}}^N | W[a,k] - W[b,k]| + \sum_{v_k \notin \cV_{0i},v_k \notin \cV_{0j}}^N | W[a,k] - W[b,k]| \nonumber\\
      &= | W[a,a] - W[b,a]| + | W[a,0j] - W[b,ji]|+ | W[a,b] - W[b,b]|  \nonumber\\
     &= (1-\beta) + |1-\cos(v_a,v_{0j})-(1-\beta)| + (1-\beta) \nonumber\\
     &= 2(1-\beta) +|\cos(v_a,v_{0j})-\beta| \le 3(1-\beta).
\end{align}

Combining  the previous equations (\ref{eq_Wab_epsilon1}), (\ref{eq_Wab_epsilon2_1}), and (\ref{eq_Wab_epsilon2_2}), we have 
\begin{align}\label{eq_Wab_epsilon}
     \|\bfW[a,:]-\bfW[b,:]\|_1
     \le \begin{cases}
          2|1-\cos(v_a,v_b)|& \text{ if } v_a, v_b \in \cV_{0i}, \forall i \in [N],\\
         3(1-\beta)& \text{ if } v_a \in \cV_{0i}, v_b \in \cV_{0j}, \forall i\neq j \in [N].
     \end{cases}.
\end{align}

\subsubsection{Proof of Lemma \ref{lemma_PiToMe_contradiction}}\label{proof_lemma_PiToMe_contradiction}
Recall the following Assumptions (A2) and (A3):
\begin{itemize}[noitemsep, nolistsep]
 \item[(A2).] There exists a margin $m$ s.t.,  $\cos(v_{a_s}, v_{b_s}) \ge m > \cos(v_{a_s}, v_{c_s}), \quad \forall v_{a_s} \in \cV^{(s)}_{0i}, \forall v_{b_s} \in \cV^{(s)}_{0i}, \forall v_{c_s} \in \cV^{(s)}_{0j}, \forall i \neq j \in [s]$.
 \vspace{0.05in}
 \item[(A3).]  There is an order of cardinality in the true partition, without loss of generality, we assume $N\ps_1 \ge N\ps_2 \ge \ldots \ge N\ps_s$, where $N\ps_i = |\cV^{(s)}_{0i}|, \forall i \in [s]$.
\end{itemize}

In \ptome approach, using Assumption (A2), given any two nodes $v_{a_s} \in \cV\ps_{0i}, v_{b_s} \in \cV\ps_{0j}, i< j \in [s]$, the energy scores $E_{a_s}$ and $E_{b_s}$ of nodes $v_{a_s}$ and $v_{b_s}$, respectively, are simplified as follows:
\setlength{\abovedisplayskip}{2pt} 
\setlength{\belowdisplayskip}{2pt} 
\begin{align} \label{equa_energy_appendix}
   E_{a_s} &= \frac{1}{N} \sum_{c_s \in \cN(v_{a_s})}  {f}_m(\cos(v_{a_s},v_{c_s})), \text{ where } {f}_m(x) = \begin{cases}
        x  & \text{ if }  x \geq m \\
        \beta & \text{ if } x < m
    \end{cases}\nonumber\\
    &= \frac{1}{N} \sum_{v_{c_s} \in \cV\ps_{0i}}  \cos(v_{a_s},v_{c_s}) + \frac{N-N\ps_i}{N}\beta \ge  \frac{N\ps_im+(N-N\ps_i)\beta}{N},\\
     E_{b_s} &= \frac{1}{N} \sum_{v_{c_s} \in \cV\ps_{0j}}  \cos(v_{b_s},v_{c_s}) + \frac{N-N\ps_j}{N}\beta \le \frac{N\ps_j+(N-N\ps_j)\beta}{N}.
\end{align}
Given the choice of the universal margin $m$ as follows: $m = \max\left\{\frac{N\ps_j}{N\ps_i}: i<j \in [s]\right\}$, it holds that $E_{a_s}\ge E_{b_s}$. Indeed, this is guaranteed as long as we have
\begin{align}
   \frac{N\ps_im+(N-N\ps_i)\beta}{N} &\ge \frac{N\ps_j+(N-N\ps_j)\beta}{N}.
\end{align}
This is equivalent that 
\begin{align}
   m &\ge \frac{N\ps_j+(N-N\ps_j)\beta-(N-N\ps_i)\beta}{N\ps_i} = \frac{N\ps_j+(N\ps_i-N\ps_j)\beta}{N\ps_i} \ge \frac{N\ps_j}{N\ps_i}.
\end{align}
Using Assumption (A3) and this choice of universal margin $m$, it holds that
\begin{align}\label{eq_margin_E}
    E_{a_s} \ge E_{b_s}, \text{ for any }v_{a_s} \in \cV\ps_{0i}, v_{b_s} \in \cV\ps_{0j}, i< j \in [s].
\end{align}

Recall that in \ptome approach, we use the ordered energy-based for the bipartite soft matching where we defined two set $\cA$ and $\cB$ with $|\cA| = |\cB| = k$ as follows:
\begin{align}\label{eq_define_A_B_appendix}
    \cA = \left\{ v^e_{1},v^e_{3},\ldots,v^e_{2k-1}\right\}, \cB = \left\{ v^e_{2},v^e_{4},\ldots,v^e_{2k}\right\},
\end{align}
where the nodes $v^e_{i}, i \in [2k]$, are sorted in decreasing order based on their energy scores, \ie $E_{v^e_{i}}> E_{v^e_{i+1}}, \forall i \in[2k-1]$.

We return to the proof of the Lemma \ref{lemma_PiToMe_contradiction} by contradiction.  Assume that $v_{a_s} \in \cV\ps_{0i}, v_{b_s} \in \cV\ps_{0j}, i\neq j \in [s]$. Without loss of generality, we assume that $v_{a_s} = v^e_{1} \in \cV\ps_{01}$ and $N\ps_1 = |\cV\ps_{01}| > 1$. 
Note that \ptome algorithm selects $v_{b_s} \in \cB$ such that
\begin{align*}
    v_{b_s} = \argmax{v_{c_s \in \cB}} \cos(v_{a_s},v_{c_s}).
\end{align*}
This is equivalent that for any $v_{c_s} \in \cB$, it holds that:
\begin{align*}
    \cos(v_{a_s},v_{c_s}) \le  \max{v_{c_s \in \cB}} \cos(v_{a_s},v_{c_s}) =  \cos(v_{a_s},v_{b_s})< m.
\end{align*}
Assumption (A2) implies that $v^e_{2i} \notin \cV\ps_{01},\forall i\in[k]$.  Since we have $N\ps_1 = |\cV\ps_{01}| > 1$, there exists at least one node in $\cA\cap \cV\ps_{01}$, say for example $v^e_3$. Using (\ref{eq_define_A_B_appendix}), we have $E_{v^e_{2}}\ge E_{v^e_{3}}$, which contradicts (\ref{eq_margin_E}) where we have $E_{v^e_{3}}> E_{v^e_{2}}$ since $v^e_{3} \in \cV\ps_{01}$ and $v^e_{2} \in \cV\ps_{0j}, j >1$.

\subsubsection{Proof of Lemma \ref{lemma_ToMe_contradiction}}\label{proof_lemma_ToMe_contradiction}
    On the contrary, in the Bipartite Soft Matching algorithm from the ToMe approach \cite{tome}, the authors divide the tokens into two distinct sets $\cA$ and $\cB$ and merge the top k similar tokens using some partitioning style like sequential, alternating and random without considering the ordered energy-based Bipartite Soft Matching like ours.
This leads to a case where the random event $M\ps_{ab}$ does not occur and therefore $\Pb(M\ps_{ab}) < 1$.
Indeed, this case happens when all the nodes from the true partition $\cV\,\ps_{01}$ are divided into the same set $\cA$. Therefore, the Bipartite Soft Matching algorithm has to choose and select the node $ v_{b_s} \in \cV\ps_{0j}, j>1 $, for merging. Actually, this case arises when all the nodes from the true partition $\cV\ps_{01}$ are distributed into the same set $\cA$. Therefore, the Bipartite Soft Matching algorithm has to select the node $v_{b_s} \in \cV\ps_{0j}, j>1$, for merging.

\section{Token Merging Outputs Visualization}\label{sec:examples}
To better illustrate the effectiveness of \ptome, we have incorporated many visualizations showcasing examples randomly sampled from the MSCOCO dataset, as depicted in Figures \ref{fig:visualize1}. In this analysis, we utilized the BLIP backbone and configured the reduction percentage to $r=0.85$. These visualizations portray the final representation of all remaining tokens in the last layers. Notably, tokens outlined with bolder \textcolor{darkcyan}{cyan} borders signify higher attention scores from the classification (CLS) token. Apparently, unlike ToMe and DiffRate, \ptome adeptly approximates the spectral spectrum of the original token space, thereby preserving the distribution of all tokens. Consequently, \ptome demonstrates a lower level of information distortion due to reduced mis-merged tokens, thus retaining attention maps for critical information in the final layer. As illustrated in Figures \ref{fig:visualize1},\ref{fig:visualize2},\ref{fig:visualize3},\ref{fig:visualize4},\ref{fig:visualize5} the attention map in the final layer of the model compressed by \ptome exhibits the closest resemblance to that of the baseline model.

\begin{figure}[t]
    \caption{Visualizations using random images sampled from the MSCOCO dataset.}
        \resizebox{\textwidth}{!}{\input{Figs/Visualizations/dog}}
        \subcaption{\textcolor{darkcyan}{A white dog} catching a novelty \textcolor{darkcyan}{flying disc} in \textcolor{darkcyan}{a competition}.}
        \label{fig:visualize1}
        \resizebox{\textwidth}{!}{\input{Figs/Visualizations/horse}}
        \subcaption{\textcolor{darkcyan}{A woman }riding a \textcolor{darkcyan}{horse} jumping it over \textcolor{darkcyan}{obstacles}.}
        \label{fig:visualize2}
        \resizebox{\textwidth}{!}{\input{Figs/Visualizations/motor}}
        \subcaption{\textcolor{darkcyan}{Three} different \textcolor{darkcyan}{motorcycle couples} riding down \textcolor{darkcyan}{a road}.}
        \label{fig:visualize3}
        \resizebox{\textwidth}{!}{\input{Figs/Visualizations/tennis}}
        \subcaption{\textcolor{darkcyan}{A man} holding \textcolor{darkcyan}{a tennis racquet} on \textcolor{darkcyan}{a tennis court.}}
        \label{fig:visualize4}
        \resizebox{\textwidth}{!}{\input{Figs/Visualizations/people}}
        \subcaption{\textcolor{darkcyan}{A group of people} who are around \textcolor{darkcyan}{a kitchen counter.}}
        \label{fig:visualize5}
\end{figure}

\section{OpenChat with Compressed LLaVA-1.5 Model}\label{sec:open-chat}
This section is dedicated to showcasing sample responses of LLaVA-1.5-7B when compressed using different algorithms. Here, we set the compression rate to $r=0.9$ for each CLIP layer. All sample answers are given in tables \ref{tab:example-baseball}, \ref{tab:example_horse}, \ref{tab:example-man-taxi}. Texts that are highlighted by \hlc[cyan!30]{blue} color indicate false information given by the model. From these sample answers, it is evident that, unlike other BSM-based approaches, our algorithm \ptome not only facilitates quicker model responses but also retains its capability to recognize objects, understand contextual information in the image, and provide correct answers closely aligned with those of the original LLaVA-1.5-7B model. 

\begin{table}[t]
    \centering
    \caption{}
    \label{tab:example-baseball}
    \resizebox{0.9\textwidth}{!}{
    \begin{tabular}{l|p{10cm}}
    \toprule
         \multicolumn{2}{c}{ \includegraphics[width=0.5\textwidth]{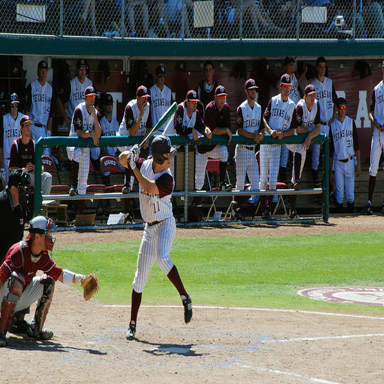} }\\
         \multicolumn{2}{c}{USER: Describe what you see}\\
         Model& Answer\\
         \midrule
         \textbf{LLaVA-1.5-7B}& In the image, a baseball game is taking place with a batter standing at home plate, holding a baseball bat and preparing to swing. The catcher is positioned behind the batter, ready to catch the ball if the batter misses. \hlc[cyan!30]{There are several other players on the field}, including teammates and opponents, all focused on the ongoing play. \hlc[cyan!30]{Some of the players are standing closer to the batter, while others are spread out across the field.} The scene captures the excitement and anticipation of a crucial moment in the game.\\
         \midrule
         \textbf{LLaVA-1.5-7B + \ptome}& In the image, a baseball player is in the middle of swinging a baseball bat, attempting to hit the ball during a game. There are several other people in the scene, \hlc[cyan!30]{some of whom are wearing baseball gloves}, indicating that they are also participating in the game. The presence of multiple players and the action of the batter suggest that this is an exciting moment in the game, with the batter trying to make contact with the ball and potentially score a run for their team.\\
         \midrule
         \textbf{LLaVA-1.5-7B +ToMe}&  In the image, a group of people is gathered on a baseball field, \hlc[cyan!30]{with some of them holding baseball bats}. It appears that they are either practicing or participating in a baseball game. \hlc[cyan!30]{There are several people standing in various positions on the field}, and some of them are closer to the foreground while others are further away. The scene captures the excitement and camaraderie often associated with baseball, as people come together to enjoy the game and improve their skills.\\
         \midrule
         \textbf{LLaVA-1.5-7B + ToFu}& In the image, there is a baseball game in progress \hlc[cyan!30]{with multiple players on the field}. One player is holding a baseball bat and is in the process of swinging it, likely attempting to hit the ball. Another player is holding a baseball glove, ready to catch the ball if it is hit. \hlc[cyan!39]{There are several other players on the field, some of them wearing baseball gloves as well}. A bench can be seen in the background, likely for players to sit on when they are not actively participating in the game. The scene captures the excitement and action of a baseball game.\\
         \midrule
         \textbf{LLaVA-1.5-7B + DiffRate} &  In the image, a group of people is playing a game of baseball on a field. \hlc[cyan!30]{There are several players on the field, with some of them holding baseball bats and wearing baseball gloves}. The players are actively engaged in the game, with one player holding a baseball bat and another player wearing a baseball glove, ready to catch the ball. \hlc[cyan!30]{The other players are also positioned around the field, either waiting for their turn to bat or preparing to field the ball}. The scene captures the excitement and energy of a live baseball game.\\
        \bottomrule
    \end{tabular}}
\end{table}

\begin{table}[t]
    \centering
    \caption{}
    \label{tab:example-man-taxi}
    \resizebox{0.9\textwidth}{!}{
    \begin{tabular}{l|p{10cm}}
    \toprule
         \multicolumn{2}{c}{ \includegraphics[width=0.5\textwidth]{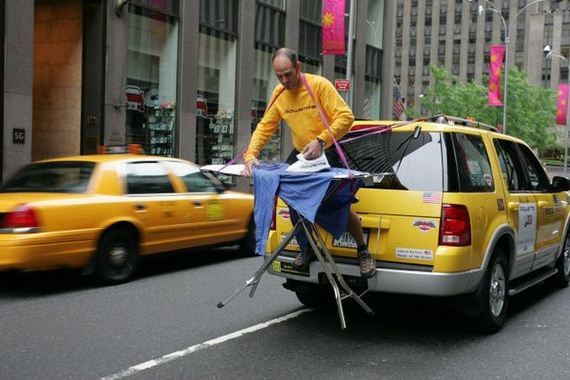} }\\
         \multicolumn{2}{c}{USER: Describe what you see.}\\
         Model& Answer\\
         \midrule
         \textbf{LLaVA-1.5-7B}&In the image, a man \hlc[cyan!30]{is standing on a folding chair} in the back of a yellow taxi cab. He is holding a blue shirt, possibly drying it or preparing to put it on. The taxi is driving down a busy street with multiple cars around it, including another taxi in front of it. The man's position on the chair and the presence of the shirt suggest that he might be a street vendor or a worker who is using the taxi as a mobile workspace.\\
         \midrule
        \textbf{ LLaVA-1.5-7B + \ptome}&  In the image, a man is standing on the back of a moving car, holding a clothesline with clothes hanging from it. The car appears to be a taxi cab, and the man is wearing a yellow shirt. The scene seems to be taking place in a busy city, as there are multiple cars visible in the background. The man's unusual activity of hanging clothes from the back of a moving car likely attracts attention and curiosity from passersby.\\
         \midrule
         \textbf{LLaVA-1.5-7B +ToMe}&  The image shows a man standing on the back of a taxi cab, \hlc[cyan!30]{holding a large umbrella}. The taxi is driving down a street, and the man appears to be \hlc[cyan!30]{using the umbrella to shield himself from the elements}. In addition to the taxi, there are several other vehicles in the scene, including \hlc[cyan!30]{a truck and a car}. The man on the taxi is the main focus of the image, and \hlc[cyan!30]{his actions suggest that he is trying to protect himself from rain or other weather conditions}.\\
         \midrule
        \textbf{ LLaVA-1.5-7B + ToFu}&The image depicts a busy city street scene with a man standing on the back of \hlc[cyan!30]{a truck}, possibly \hlc[cyan!30]{performing a stunt or participating in a parade}. The truck is positioned in the middle of the street, and the man appears to be the main focus of the scene. \hlc[cyan!30]{There are several other people in the image, some of whom are walking or standing near the truck}. Additionally, there are multiple cars parked or driving along the street, and \hlc[cyan!30]{a traffic light can be seen in the background}. The overall atmosphere suggests an urban environment with a lively and bustling atmosphere.\\
         \midrule
         \textbf{LLaVA-1.5-7B + DiffRate} &  The image features a man wearing a yellow shirt, \hlc[cyan!30]{standing next to a yellow car}. He appears to be \hlc[cyan!30]{holding onto the car}, possibly for support or to steady himself. The car is parked on a street, and there are other cars visible in the background. In addition to the man and the car, \hlc[cyan!30]{there are two other people in the scene, one standing further back and another closer to the foreground}. \hlc[cyan!30]{A handbag can also be seen placed on the ground near the center of the image.}\\
        \bottomrule
    \end{tabular}}
\end{table}

\begin{table}[t]
    \centering
    \caption{}
    \label{tab:example_horse}
    \resizebox{0.9\textwidth}{!}{
    \begin{tabular}{l|p{10cm}}
    \toprule
         \multicolumn{2}{c}{ \includegraphics[width=0.5\textwidth]{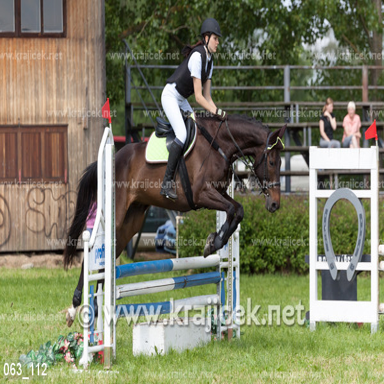} }\\
         \multicolumn{2}{c}{USER: Describe what is happening?}\\
         Model& Answer\\
         \midrule
         \textbf{LLaVA-1.5-7B}&In the image, a woman is riding a horse and jumping over a hurdle, which is a part of an equestrian competition. The horse and rider are in the middle of the jump, showcasing their skills and coordination. There are several other people in the scene, some of whom may be fellow competitors or spectators. The presence of multiple people and the focus on the horse and rider suggests that this is a competitive event or practice session.\\
         \midrule
        \textbf{ LLaVA-1.5-7B + \ptome}& In the image, a person is riding a horse while jumping over an obstacle, likely during an equestrian event or competition. The horse and rider are in the middle of the jump, showcasing their skill and coordination. There are several other people in the scene, some of whom might be fellow competitors, trainers, or spectators. \hlc[cyan!30]{One person is standing close to the rider on the horse}, while others are scattered around the area. The overall atmosphere of the image suggests an exciting and challenging equestrian event.\\
         \midrule
         \textbf{LLaVA-1.5-7B +ToMe}& In the image, a person is riding a horse and jumping over an obstacle, \hlc[cyan!30]{possibly a hurdle or a small hill}. The horse and rider are in the middle of the action, showcasing their skill and coordination. The scene captures the excitement and challenge of equestrian sports, such as show jumping or eventing.\\
         \midrule
        \textbf{ LLaVA-1.5-7B + ToFu}& In the image, a \hlc[cyan!30]{person is riding a horse while standing in a corral}. \hlc[cyan!30]{The horse and rider are positioned near a fence}, \hlc[cyan!30]{which could be a part of a fenced-in area or a stable}. The scene captures the bond between the rider and the horse as they interact and move together \hlc[cyan!30]{within the corral.}\\
         \midrule
         \textbf{LLaVA-1.5-7B + DiffRate} &  In the image, \hlc[cyan!30]{there are two people riding horses}, likely participating in a horseback riding lesson or practice session. \hlc[cyan!30]{The horses and riders are positioned next to each other, with one horse and rider on the left side and the other horse and rider on the right side of the scene}. The riders are wearing helmets for safety, and the horses appear to be well-trained and cooperative. The scene captures the essence of a horseback riding activity, where \hlc[cyan!30]{the riders are learning to control and communicate with their horses effectively.}\\
        \bottomrule
    \end{tabular}}
\end{table}

\end{document}